\documentclass[letterpaper,10pt]{article}
\usepackage[utf8]{inputenc}

\PassOptionsToPackage{compatibility=false}{caption}
\PassOptionsToPackage{hypertexnames=false}{hyperref}

\usepackage[preprint]{probml}
\ShortHeadings{Variational Kernel Design for Internal Noise}

\usepackage{microtype}
\usepackage{graphicx}
\usepackage{algorithm}
\usepackage[noend]{algpseudocode}
\usepackage{placeins}
\usepackage{enumitem}
\usepackage{stmaryrd}
\usepackage{mathtools}

\numberwithin{theorem}{section}

\newcommand{\E}{\mathbb{E}}
\newcommand{\Var}{\mathrm{Var}}
\newcommand{\Cov}{\mathrm{Cov}}
\newcommand{\R}{\mathbb{R}}

\newcommand{\KL}{\mathrm{KL}}
\newcommand{\tr}{\mathrm{Tr}}
\newcommand{\eps}{\varepsilon}

\newcommand{\GCh}{{\normalfont\textsc{GCh}}}

\begin{document}

\title{Variational Kernel Design for Internal Noise:\\
Gaussian Chaos Noise, Representation Compatibility, and Reliable Deep Learning
}

\author[1,2,3,$\dagger$]{Ziran Liu}

\affil[1]{Shanghai Institute for Mathematics and Interdisciplinary Sciences (SIMIS), Shanghai 200433, China}
\affil[2]{Research Institute of Intelligent Complex Systems, Fudan University, Shanghai 200433, China}
\affil[3]{Institute for Intelligent Computing@SJTU, Shanghai 200433, China}
\affil[$\dagger$]{Correspondence to \url{zliu@simis.cn}}

\maketitle

\begin{abstract}
Internal noise in deep networks is usually inherited from heuristics such as dropout, hard masking, or additive perturbation. We ask two questions: what correlation geometry should internal noise have, and is the implemented perturbation compatible with the representations it acts on? We answer these questions through Variational Kernel Design (VKD), a framework in which a noise mechanism is specified by a law family, a correlation kernel, and an injection operator, and is derived from learning desiderata. In a solved spatial subfamily, a quadratic maximum-entropy principle over latent log-fields yields a Gaussian optimizer with precision given by the Dirichlet Laplacian, so the induced geometry is the Dirichlet Green kernel. Wick normalization then gives a canonical positive mean-one gate, Gaussian Chaos Noise (\GCh). For the sample-wise gate used in practice, we prove exact Gaussian control of pairwise log-ratio deformation, margin-sensitive ranking stability, and an exact expected intrinsic roughness budget; hard binary masks instead induce singular or coherence-amplified distortions on positive coherent representations. On ImageNet and ImageNet-C, \GCh{} consistently improves calibration and under shift also improves NLL at competitive accuracy.
\end{abstract}


\section{Introduction}
\label{sec:intro}

Noise injection is one of the most widely used yet least principled components of deep learning. It appears as additive perturbation, stochastic gating, masking, augmentation, corruption-aware training, and uncertainty regularization, and it is routinely used to improve generalization, calibration, and robustness. Yet one central design choice is usually left heuristic: \emph{what structure should the noise have}? In most pipelines, that choice is inherited from familiar templates such as i.i.d.\ dropout \citep{srivastava2014dropout}, stochastic depth \citep{huang2016stochasticdepth}, or hard spatial masking \citep{ghiasi2018dropblock}, rather than derived from the geometry of the representation or the objective of the learner.

This raises a more structural question:
\begin{quote}
\emph{If internal noise is to be used as part of representation learning, what aspects of that noise should be derived from first principles rather than fixed by convention?}
\end{quote}

Our answer is to treat internal noise as a \emph{design object}. We call the resulting program \emph{Variational Kernel Design} (VKD). In VKD, a noise mechanism is specified by a triple
\[
N=(\mathcal{F},K,\mathcal{T}),
\]
consisting of a law family, a correlation kernel, and an injection operator. A realization map then turns a sampled latent field into an implemented perturbation. The mechanism is therefore not just ``a distribution''; it is a compositional system that separates \emph{what is sampled}, \emph{what geometry it must respect}, and \emph{where and how it is deployed}.

This viewpoint reveals that there are really two linked questions. The first is a \emph{design question}: once locality, smoothness, and mean-preserving positivity are encoded as operator-level constraints, what perturbation geometry is canonically induced? The second is a \emph{compatibility question}: once such a mechanism is implemented in a deep network, what does it actually do to the geometry of positive semantic representations, and how does that differ from hard masking? The paper is built around this two-layer split. The first layer derives the mechanism. The second studies the induced action of the implemented perturbation on a target representation regime.

The design layer leads to a solved quadratic VKD program. In the spatial setting of this paper, a maximum-entropy log-field under a Dirichlet-energy budget is Gaussian with precision $\beta L_U$, hence covariance $\beta^{-1}L_U^{-1}$. In other words, within the chosen local quadratic design class, the Dirichlet Green kernel is not an additional modeling choice; it is the inverse operator forced by the constraints. Exponentiating that field with exact Wick normalization yields a positive mean-one multiplicative gate, which we call \GCh{}.

The compatibility layer is where the practical distinction emerges. For the sample-wise gate actually used in our experiments, we prove exact Gaussian control of pairwise log-ratio deformations, explicit margin-sensitive ranking stability, and an exact expected intrinsic roughness budget. For hard binary masks, we prove a qualitatively different behavior: incompatibility with finite log-ratio geometry, a margin-blind ranking law for inverted dropout, and a coherence-sensitive distortion term whose relative size diverges as the underlying representation becomes increasingly smooth. This is the rigorous form of the informal claim that smooth positive multiplicative perturbations are better matched to coherent late-stage semantics than hard deletion.

A key theme throughout the paper is that these two layers belong together. The contribution is not just that a Gaussian field can be derived from a quadratic maximum-entropy problem---that isolated fact is classical. The contribution is that the variational solution is used as an operator-level design map for training-time noise, \emph{and then} analyzed as an implemented mechanism acting on coherent positive evidence maps. In short: first derive the geometry; then ask whether the realized perturbation is compatible with the representation regime of interest.

\paragraph{Scope of the theoretical claims.}
The paper does \emph{not} claim that deeper is always better, or that hard masking is universally inferior on every architecture, layer, or objective. The claims are conditional and operational: when a layer carries positive region-level or token-level evidence and becomes increasingly coherent in the late-semantic sense, the mathematically relevant quantities are relative log-ratios, ranking stability, and aggregate geometric roughness. In that regime, we show that the implemented \GCh{} gate yields finite, margin-aware Gaussian deformations, whereas hard binary masks yield singular or coherence-amplified distortions.

This perspective yields both a principled mechanism and a practical prediction. If later layers encode increasingly decisive relative evidence between regions or tokens while also becoming more spatially coherent, then a margin-aware smooth multiplicative gate should remain compatible with those representations, whereas hard masking should become increasingly mismatched. Our experiments are designed to test exactly this distinction. On clean ImageNet, \GCh{} improves calibration substantially, and on the selected 7-corruption ImageNet-C evaluation it improves both ECE and NLL while maintaining competitive accuracy. It also remains effective in late-stage injection settings where hard masking can degrade clean calibration.

\paragraph{Contributions.}
Our contributions are as follows.
\begin{itemize}[leftmargin=1.4em]
  \item \textbf{A framework view of internal noise.} We formulate internal noise injection as a compositional design problem and introduce VKD, in which a mechanism is derived from learning-motivated constraints rather than selected from a fixed menu of perturbations.
  \item \textbf{A two-layer theory: design and compatibility.} We separate a mechanism-design layer from a representation-compatibility layer, making explicit the distinction between what is derived from first principles, what is realized in implementation, and what is subsequently measured on a target representation regime.
  \item \textbf{A solved quadratic MaxEnt design program.} We state the admissible class of centered log-field laws explicitly, solve the resulting finite-dimensional variational problem in closed form, and derive an entropy-gap identity certifying uniqueness of the optimizer.
  \item \textbf{Operator-forced kernel geometry.} For spatial log-fields with a Dirichlet-energy budget and gauge fixing, the optimizer is Gaussian with covariance proportional to the Dirichlet Green kernel. More generally, replacing the quadratic operator replaces the induced kernel by its inverse.
  \item \textbf{A canonical exact gate and an implementation-aware framework.} Exponentiating the MaxEnt log-field with Wick normalization yields \GCh{}, a positive mean-one multiplicative gate with explicit multi-point moments; once the operator and budget are fixed, the exact gate becomes an effectively one-parameter family through $\tau=\gamma^2/\beta$. We also make explicit the split between the canonical exact gate and the sample-wise implementation used in practice.
  \item \textbf{Representation compatibility versus hard-mask mismatch.} For the sample-wise gate used in the experiments, we prove exact Gaussian control of pairwise log-ratios, margin-sensitive ranking stability, and an exact expected intrinsic roughness budget. For hard binary masking, we prove incompatibility with finite log-ratio geometry, a margin-blind ranking law for inverted dropout, an immediate loss-of-perfect-coherence result in expectation on perfectly coherent maps, and a late-stage mismatch theorem in the coherent-representation regime.
  \item \textbf{Empirical validation in the predicted late-stage regime.} On clean ImageNet, a selected 7-corruption ImageNet-C evaluation, Swin-T, and a fine-grained Oxford-IIIT Pets pilot, \GCh{} improves calibration and, under shift, also improves NLL, all at competitive accuracy. Controlled ablations show the importance of correlation, positivity, and injection depth.
\end{itemize}

\paragraph{A practical way to read the paper.}
The variational results explain \emph{where the kernel comes from}. The compatibility results explain \emph{why the resulting implemented gate behaves differently from binary masking on semantic representations}. The experiments then test that distinction precisely in the late-stage regime where the mismatch should matter most.

\paragraph{Roadmap.}
Section~\ref{sec:related} reviews stochastic regularization, calibration, and robustness under shift. Section~\ref{sec:framework} presents VKD as a compositional design system and situates the paper's solved instance inside that framework. Sections~\ref{sec:method}--\ref{sec:gch} develop the Dirichlet log-field construction, the quadratic MaxEnt theorem, and the exact and implemented \GCh{} gates. Section~\ref{sec:geometry_theory} gives the representation-compatibility analysis, and Section~\ref{sec:experiments} tests the resulting predictions empirically.

\paragraph{Paper in one sentence.}
We derive the noise geometry from first principles and then show that the resulting implemented smooth positive gate preserves finite, margin-aware relative geometry exactly in the regime where hard masking becomes singular or coherence-amplified.

\paragraph{What is classical and what is new.}
The isolated fact that quadratic maximum entropy yields a Gaussian law is classical. The contribution here is the \emph{use} of that principle as an operator-level design map for training-time noise, together with the second layer of theory that is specific to this paper: exact representation-compatibility results for the implemented sample-wise gate and exact incompatibility results for hard binary masks on coherent positive semantic representations. Put differently, the variational theorem identifies the canonical kernel inside a chosen design class, and the later compatibility theorems explain why that designed mechanism behaves differently from masking in deep networks.

\section{Related Work}
\label{sec:related}

\paragraph{Noise injection and regularization in deep networks.}
Small additive noise is classically linked to Tikhonov-style regularization \citep{bishop1995training}. Dropout injects i.i.d.\ Bernoulli gating \citep{srivastava2014dropout}; stochastic depth drops residual branches \citep{huang2016stochasticdepth} and is extended to Transformers via LayerDrop \citep{fan2019layerdrop}; ShakeDrop perturbs residual branches with randomized coefficients \citep{yamada2018shakedrop}. Spatial occlusion methods such as Cutout and DropBlock impose structured hard masking on feature maps \citep{devries2017cutout,ghiasi2018dropblock}, while sample-level mixing methods such as Mixup and CutMix inject stochasticity at the data level \citep{zhang2018mixup,yun2019cutmix}. In vision transformers, PatchDropout removes input patches and changes token topology \citep{liu2023patchdropout}. A common limitation is that the correlation structure of the noise is usually fixed a priori and often assumes spatial independence or hard discontinuities, which can mismatch late semantic representations.

\paragraph{Calibration and reliability under shift.}
Miscalibration is widespread in modern neural networks, and temperature scaling remains a strong post-hoc baseline \citep{guo2017calibration}. Nonparametric alternatives include BBQ \citep{naeini2015bbq}, while Dirichlet calibration extends beyond a single temperature parameter \citep{kull2019dirichlet}. Dropout admits an approximate Bayesian interpretation \citep{gal2016dropoutbayes}, and deep ensembles remain a strong uncertainty baseline \citep{lakshminarayanan2017ensembles}. Under distribution shift, calibration can deteriorate substantially \citep{ovadia2019trust}, and recent work emphasizes that calibration depends strongly on architecture and training recipe \citep{minderer2021revisiting}. Label smoothing can help but is context dependent \citep{muller2019labelsmoothing}. These findings motivate methods that improve NLL and ECE directly during representation learning rather than relying only on post-hoc correction.

\paragraph{Robustness to corruptions and distribution shift.}
For worst-case robustness, adversarial training and TRADES formalize the robustness--accuracy trade-off \citep{madry2018towards,zhang2019trades}. For average-case corruptions, ImageNet-C/P provide standardized benchmarks \citep{hendrycks2019imagenetc}; subsequent work has also emphasized that performance on synthetic corruptions does not perfectly transfer to natural shifts \citep{taori2020measuring}, and broader OOD suites reveal substantial heterogeneity across shift types \citep{hendrycks2021manyfaces}. Simple augmentation policies such as RandAugment and AugMix improve corruption robustness and uncertainty with low overhead \citep{cubuk2020randaugment,hendrycks2020augmix}; properly tuned Gaussian or speckle noise can also be effective \citep{rusak2020simple}. Noisy Student further demonstrates the power of strong stochastic regularization in large-scale training \citep{xie2020noisystudent}. Our focus is complementary: rather than designing perturbations at the input level, we derive an \emph{internal} spatial noise mechanism whose correlation structure follows from explicit desiderata.

\section{Variational Kernel Design as a Compositional Design System}
\label{sec:framework}

We treat internal noise not as a fixed perturbation template but as a mechanism to be derived from learning desiderata. The role of Variational Kernel Design (VKD) is to map a collection of task-level constraints to a stochastic mechanism and then to analyze how that mechanism acts on a target representation regime. This viewpoint separates two layers that are often conflated in practice: a \emph{mechanism-design layer}, which specifies what latent object is sampled, what geometry it must respect, and where it is injected, and a \emph{compatibility layer}, which studies what geometric quantities the deployed perturbation preserves or distorts on the representations actually used by the network.

The benefit of this separation is conceptual as well as practical. It makes clear which parts of the construction are derived from first principles, which parts are implementation choices, and which parts are properties of the resulting perturbation on a given representation regime. In particular, VKD is not a menu of named noises; it is a compositional system for deriving, realizing, deploying, and analyzing an internal perturbation mechanism.

\subsection{Mechanism space: VKD as a design system}

Let $\Omega$ denote a perturbation domain and let $\mathcal H$ denote a feature space. A VKD mechanism is specified by a triple
\[
N=(\mathcal F,K,\mathcal T),
\]
whose three components encode complementary axes of design.

\begin{definition}[VKD mechanism]
A \emph{VKD mechanism} on $(\Omega,\mathcal H)$ is a triple
\[
N=(\mathcal F,K,\mathcal T),
\]
where:
\begin{enumerate}[label=\textbf{(\roman*)},leftmargin=1.4em]
    \item $\mathcal F$ is a family of laws on latent fields $\psi\in\R^\Omega$;
    \item $K$ is a positive semidefinite kernel on $\Omega\times\Omega$ encoding the intended second-order geometry;
    \item $\mathcal T$ is an injection operator that deploys a realized perturbation inside the model.
\end{enumerate}
\end{definition}

The three components play distinct roles. The family $\mathcal F$ determines what latent object is sampled; the kernel $K$ encodes how that object is spatially correlated; and the operator $\mathcal T$ determines where and how the realized perturbation acts on the network. In this way, VKD separates \emph{sampling}, \emph{geometry}, and \emph{deployment}.

To make the construction operational, we introduce a realization map
\[
\ell:\R^\Omega\to (0,\infty)^\Omega,
\]
which turns a latent field $\psi$ into a positive gate $\xi=\ell(\psi)$. The deployed perturbation is then
\[
\widetilde h = \mathcal T(h;\xi),
\qquad
\psi\sim \mathcal F.
\]
Thus the mechanism pipeline has the schematic form
\[
(\mathcal F,K,\mathcal T)
\quad\Longrightarrow\quad
\psi\sim\mathcal F
\quad\xrightarrow{\ \ell\ }\quad
\xi
\quad\xrightarrow{\ \mathcal T\ }\quad
\widetilde h.
\]

\subsection{From desiderata to admissible mechanism classes}

A central point of VKD is that the mechanism is not selected from a fixed heuristic menu. Instead, one starts from a collection of learning desiderata $D$---for example positivity, lack of systematic scale drift, locality, smoothness, or minimal extra information---and translates them into mathematical constraints on admissible mechanisms.

Accordingly, VKD should be read as a map
\[
D \quad\longmapsto\quad \mathfrak N(D),
\]
where $\mathfrak N(D)$ is an admissible class of mechanisms consistent with the desiderata. The design problem is then to derive a distinguished mechanism
\[
N^\star \in \mathfrak N(D)
\]
rather than choose one by convention.

This formulation is intentionally general. In some settings, the admissible class may leave several components independent. In other settings, the desiderata may couple the law and the geometry so strongly that the kernel is no longer a free modeling knob but a derived consequence of the design class itself.

\subsection{A two-layer view: mechanism and compatibility}

A VKD mechanism is only half of the story. Once a mechanism has been derived and realized, one must still ask how the deployed perturbation acts on the representations the network actually uses. We therefore separate a second object: a target representation regime $\mathcal R$ together with a collection of compatibility observables
\[
\mathcal O(\widetilde h;\mathcal R),
\]
such as pairwise log-ratio deformation, ranking stability, intrinsic roughness inflation, or topological stability.

The resulting conceptual split is:
\begin{itemize}[leftmargin=1.5em]
    \item \textbf{Mechanism-design layer:} derive $(\mathcal F,K,\mathcal T)$ and the realization map $\ell$ from desiderata;
    \item \textbf{Compatibility layer:} study the induced action of the deployed mechanism on observables relevant to a target representation regime $\mathcal R$.
\end{itemize}

This distinction is especially important in the present paper because the canonical object derived by the variational theory is an exact Wick-normalized gate, while the optimization-friendly implementation used in the main experiments is a sample-wise mean-one gate. The design layer tells us \emph{what the canonical latent geometry is}; the compatibility layer tells us \emph{what the implemented mechanism does once deployed}.

\begin{remark}[VKD is compositional, not temporal]
VKD should not be interpreted as a temporal dynamical system unless an explicit update rule is introduced. Its role here is compositional: desiderata define an admissible class, the variational principle derives a canonical latent law and geometry, the realization map produces a deployed gate, and the compatibility layer studies the induced action of that gate on a target representation regime.
\end{remark}

\subsection{The solved instance studied in this paper}

The present paper studies a solved quadratic VKD subfamily in which the latent object is a centered log-field and spatial coherence is imposed through a quadratic operator budget. In this subfamily, the law and the geometry are not independent design axes: once the operator $Q$ and the energy budget $\eps$ are fixed, the unique maximum-entropy optimizer is Gaussian with covariance proportional to $Q^{-1}$. Thus the kernel is \emph{operator-forced} rather than chosen heuristically.

In the main spatial construction, the perturbation domain is the feature grid $U$, the operator is the Dirichlet Laplacian $Q=L_U$, and the resulting latent law is a discrete Gaussian free field with covariance proportional to the Dirichlet Green kernel $G_U=L_U^{-1}$. A realization map then turns the latent log-field into either a canonical exact Wick-normalized gate or the sample-wise mean-one gate used in the experiments. The injected perturbation is studied precisely in the late-stage positive coherent regime where pairwise log-ratios, ranking preservation, and intrinsic roughness are the relevant observables.

\begin{figure}[t]
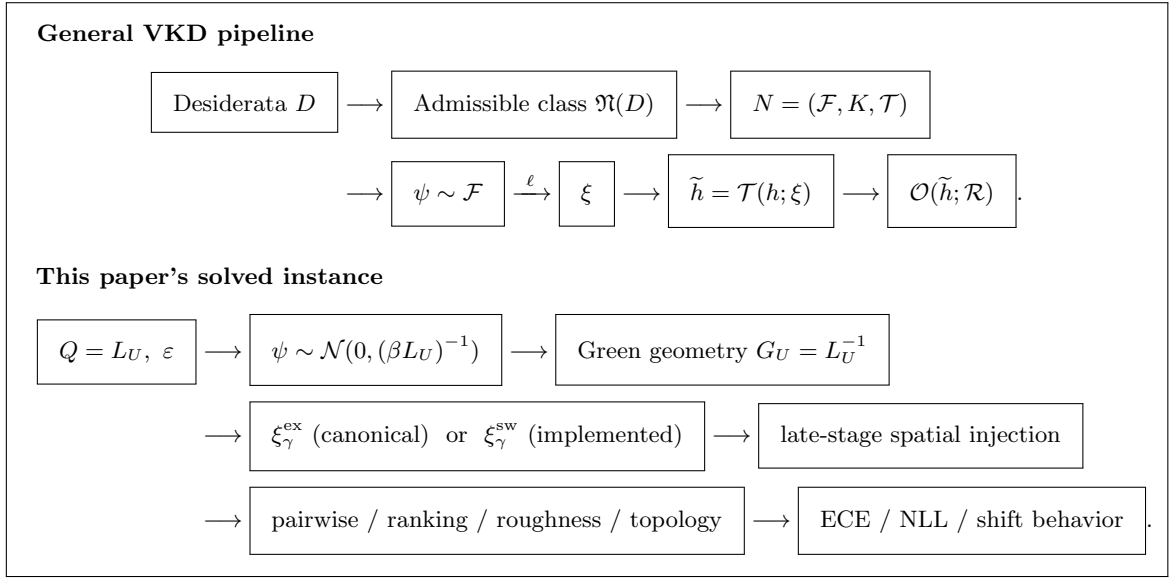

\centering
\setlength{\fboxsep}{8pt}
\fbox{
\begin{minipage}{0.96\linewidth}
\small
\textbf{General VKD pipeline}
\[
\begin{aligned}
\boxed{\text{Desiderata }D}
&\longrightarrow
\boxed{\text{Admissible class }\mathfrak N(D)}
\longrightarrow
\boxed{N=(\mathcal F,K,\mathcal T)}\\[0.2em]
&\longrightarrow
\boxed{\psi\sim\mathcal F}
\xrightarrow{\ \ell\ }
\boxed{\xi}
\longrightarrow
\boxed{\widetilde h=\mathcal T(h;\xi)}
\longrightarrow
\boxed{\mathcal O(\widetilde h;\mathcal R)}.
\end{aligned}
\]

\vspace{0.6em}
\textbf{This paper's solved instance}
\[
\begin{aligned}
\boxed{Q=L_U,\ \eps}
&\longrightarrow
\boxed{\psi\sim\mathcal N(0,(\beta L_U)^{-1})}
\longrightarrow
\boxed{\text{Green geometry }G_U=L_U^{-1}}\\[0.2em]
&\longrightarrow
\boxed{\xi_\gamma^{\mathrm{ex}}\ \text{(canonical)}\ \ \text{or}\ \ \xi_\gamma^{\mathrm{sw}}\ \text{(implemented)}}
\longrightarrow
\boxed{\text{late-stage spatial injection}}\\[0.2em]
&\longrightarrow
\boxed{\text{pairwise / ranking / roughness / topology}}
\longrightarrow
\boxed{\text{ECE / NLL / shift behavior}}.
\end{aligned}
\]
\end{minipage}
}
\caption{\textbf{VKD as a compositional design system.} The top row shows the general design logic: learning desiderata define an admissible mechanism class, from which a mechanism and realization map are derived and then evaluated through compatibility observables on a target representation regime. The bottom row shows the solved instance of this paper. A quadratic operator budget with $Q=L_U$ yields a Gaussian log-field with covariance proportional to the Dirichlet Green kernel. This latent field admits a canonical exact realization through Wick normalization and an optimization-friendly implemented realization through sample-wise mean-one normalization. The resulting deployed gate is analyzed through pairwise, ranking, roughness, and topological observables, and then tested empirically through calibration and reliability metrics.}
\label{fig:vkd_system}
\end{figure}

\paragraph{Roadmap from framework to results.}
Section~\ref{sec:method} fixes the perturbation domain and deployment axis for the spatial setting studied in this paper. Section~\ref{sec:gch} then solves the mechanism-design layer: desiderata define an admissible class, the quadratic MaxEnt principle derives the canonical latent law, and the realization map yields the exact \GCh{} gate together with the implementation-aware variants used in practice. The later subsections of Section~\ref{sec:gch} begin the compatibility layer by analyzing the induced action of the implemented gate on the positive coherent regime relevant to late-stage representations.

\section{A Solved Quadratic VKD Instance: Problem Setup and Discrete GFF Background}
\label{sec:method}

We now instantiate the general system of Section~\ref{sec:framework} in the spatial setting used throughout the paper. To keep the framework explicit, it is useful to separate what is fixed in this section from what is derived in Section~\ref{sec:gch}. Here we fix the perturbation domain $\Omega=U$, the feature space $\mathcal H=\R^{C\times H\times W}$, and the deployment axis $\mathcal T$ as spatial multiplicative injection on a feature grid. In Section~\ref{sec:gch} we then derive the canonical latent law and its induced geometry from the learning desiderata. In the solved instance studied here, the latent object is a centered log-field, the operator budget is the Dirichlet energy, and the resulting canonical geometry is the Dirichlet Green kernel.

\subsection{Injection site and spatial gating}
\label{sec:injection-context}

Fix a layer at which a feature map is perturbed. Let
\[
h \in \R^{C\times H\times W}
\]
denote the feature tensor at that site, with channel index $c\in\{1,\dots,C\}$ and spatial location $x=(i,j)\in U$, where
\[
U=\{1,\dots,H\}\times\{1,\dots,W\}.
\]
We focus on \emph{spatial} perturbations: a random field acts on the $H\times W$ grid and is shared across channels. Concretely, we introduce a positive spatial gate
\[
\nu:U\to(0,\infty),
\]
and apply it identically across channels.

\paragraph{Injection operators.}
The basic multiplicative operator is
\begin{equation}
\mathcal{T}_{\nu}(h)(c,x)=h(c,x)\,\nu(x),
\label{eq:mul_injection}
\end{equation}
that is, pointwise multiplication with spatial broadcasting. For numerical stability or reduced perturbation strength, we may also use the residual form
\begin{equation}
\mathcal{T}_{\nu}^{\mathrm{res}}(h)(c,x)
=
h(c,x)\Big(1+\alpha(\nu(x)-1)\Big),
\qquad \alpha\in(0,1].
\label{eq:residual_mul}
\end{equation}
Unless otherwise stated, we use $\alpha=1$.

\paragraph{Framework instantiation of the deployment axis.}
In the notation of Section~\ref{sec:framework}, this subsection fixes the deployment part of the mechanism: the perturbation domain is the interior grid $U$, the feature space is $\mathcal H=\R^{C\times H\times W}$, and the admissible deployment operators are spatial multiplicative injections such as $\mathcal T_\nu$ and $\mathcal T^{\mathrm{res}}_\nu$. What remains open at this stage is the design layer: which latent law should be sampled, and what correlation geometry should it induce?

\subsection{Discrete Gaussian free field on a rectangular grid}
\label{sec:dgff}

To make the implementation and spectral formulas consistent, we treat the feature grid itself as the interior domain and impose Dirichlet conditions on an \emph{auxiliary outer boundary}. Fix integers $H,W\ge 1$ and define
\[
U=\{1,\dots,H\}\times\{1,\dots,W\},
\qquad
\bar U=\{0,\dots,H+1\}\times\{0,\dots,W+1\},
\]
with auxiliary boundary
\[
B=\bar U\setminus U.
\]
Equip $\bar U$ with the nearest-neighbor undirected edge set
\[
E=\big\{\{x,y\}\subset \bar U:\ \|x-y\|_1=1\big\}.
\]
Optionally, allow positive symmetric edge weights $c_{xy}=c_{yx}>0$ on $\{x,y\}\in E$; the unweighted case is $c_{xy}\equiv 1$.

A field is a function $\phi:U\to\R$. We extend it by zero to the auxiliary boundary:
\[
\bar\phi(y)=
\begin{cases}
\phi(y), & y\in U,\\
0, & y\in B.
\end{cases}
\]

\paragraph{Dirichlet Laplacian and energy.}
For $\phi:U\to\R$, define the Dirichlet Laplacian $L_U$ by
\begin{equation}
(L_U\phi)(x)
=
\sum_{y:\{x,y\}\in E} c_{xy}\big(\phi(x)-\bar\phi(y)\big),
\qquad x\in U.
\label{eq:dirichlet_laplacian}
\end{equation}
Its quadratic form is the Dirichlet energy
\begin{equation}
\mathcal{E}(\phi)
\coloneqq
\frac12\langle \phi,L_U\phi\rangle
=
\frac12\sum_{\{x,y\}\in E} c_{xy}\big(\bar\phi(x)-\bar\phi(y)\big)^2.
\label{eq:dirichlet-energy}
\end{equation}
Under Dirichlet boundary conditions, $L_U$ is symmetric positive definite, so $\mathcal{E}(\phi)>0$ for $\phi\neq 0$.

\paragraph{Discrete GFF.}
Fix an inverse-temperature parameter $\beta>0$. The Dirichlet discrete Gaussian free field (GFF) on $U$ is the centered Gaussian vector
\begin{equation}
\phi\sim\mathcal{N}\!\big(0,(\beta L_U)^{-1}\big).
\label{eq:gff-law}
\end{equation}
Equivalently, its density on $\R^{U}$ is
\begin{equation}
p_\beta(\phi)
=
\frac{1}{Z_\beta}
\exp\!\big(-\beta\,\mathcal{E}(\phi)\big)
=
\Big(\frac{\det(\beta L_U)}{(2\pi)^{|U|}}\Big)^{1/2}
\exp\!\Big(-\tfrac12\phi^\top(\beta L_U)\phi\Big),
\label{eq:gff-density}
\end{equation}
with normalizing constant
\begin{equation}
Z_\beta=(2\pi)^{|U|/2}\det(\beta L_U)^{-1/2}.
\label{eq:gff-normalizer}
\end{equation}

\paragraph{Green kernel.}
Define the Dirichlet Green matrix
\[
G_U\coloneqq L_U^{-1}.
\]
Then the covariance of the GFF is
\begin{equation}
\Cov\big(\phi(x),\phi(y)\big)=\beta^{-1}G_U(x,y),
\qquad x,y\in U.
\label{eq:gff-cov}
\end{equation}

\paragraph{Framework role of this subsection.}
At this point the state space of latent fields and the local operator have been fixed. The next section will solve the design layer inside this spatial VKD class: the variational principle will determine the canonical law family $\mathcal F^\star$, and the induced second-order geometry will appear as a consequence of the chosen operator rather than as an additional hyperparameter.


\section{Solving the Design Layer: From Desiderata to Gaussian Chaos Noise}
\label{sec:gch}

We now solve the mechanism-design layer of VKD for the spatial instantiation fixed in Section~\ref{sec:method}. The logical order is: specify learning desiderata, define the admissible class of latent laws, derive the canonical law and induced geometry, and only then choose a realization map that turns the latent object into a deployed gate. Read in this way, the section is not only about one new noise family; it is the full derivation of a solved VKD instance. The key mathematical point is that the optimization is performed over \emph{laws of the latent log-field}; positivity and mean preservation are imposed afterwards at the realization stage through an exponential link and Wick normalization.

\subsection{Design desiderata}

Each desideratum constrains a different part of the framework: D1 selects the law inside an admissible class, D2--D3 constrain the realization map, D4 determines the operator geometry, and D5 ensures that the operator-level design problem is well posed.

\begin{itemize}[leftmargin=1.5em]
\item[\textbf{D1}] \textbf{Least additional information (maximum entropy).}
Among all admissible laws satisfying the required constraints, choose the one with maximum differential entropy. Intuitively, the perturbation should avoid injecting unintended semantics.

\item[\textbf{D2}] \textbf{Positivity through an exponential link.}
The gate should modulate amplitude without introducing sign flips or hard artifact patterns. We therefore write
\[
\xi=\exp(\zeta)
\]
for a real-valued log-field $\zeta\in\R^U$.

\item[\textbf{D3}] \textbf{No systematic scale drift.}
The gate should not create a persistent gain shift. In the exact construction this is enforced by Wick normalization, giving $\E[\xi(x)]=1$ for every site $x\in U$.

\item[\textbf{D4}] \textbf{Spatial coherence via a quadratic smoothness budget.}
The perturbation should be spatially coherent rather than pixelwise i.i.d. We encode this through a local quadratic budget on the log-field:
\begin{equation}
\E\!\left[\frac12\langle \psi,Q\psi\rangle\right]=\eps,
\label{eq:general_budget}
\end{equation}
where $Q\succ 0$ is a symmetric positive definite operator on $\R^U$. In the canonical grid construction of this paper, $Q=L_U$ is the Dirichlet Laplacian.

\item[\textbf{D5}] \textbf{Well-posedness through gauge fixing.}
A gauge convention is required so that the quadratic operator is invertible. In the main text we impose auxiliary Dirichlet boundary conditions, which make $L_U\succ 0$.
\end{itemize}

\paragraph{Why separate $\zeta$ and $\psi$?}
For the variational problem, the object being optimized is the law of a centered log-field $\psi$. Positivity and mean preservation are then enforced \emph{afterwards} by mapping $\psi$ through a Wick-normalized exponential. This separation is useful because it makes clear which parts of the theory characterize the optimizer of the entropy problem and which parts define the final multiplicative gate.

\subsection{A formal variational class}
\label{sec:formal_variational_class}

Fix an SPD operator $Q$ on $\R^U$, an energy budget $\eps>0$, and let $n\coloneqq |U|$. Define the admissible class
\begin{equation}
\mathcal{A}(Q,\eps)
\coloneqq
\left\{
p:\R^U\to[0,\infty)\ \middle|\ 
\begin{array}{l}
\int_{\R^U} p(\psi)\,d\psi=1,\\[2pt]
\int_{\R^U} \psi\,p(\psi)\,d\psi=0,\\[2pt]
\int_{\R^U} \frac12\langle \psi,Q\psi\rangle p(\psi)\,d\psi=\eps,\\[2pt]
h(p)>-\infty
\end{array}
\right\},
\label{eq:admissible_class}
\end{equation}
where
\[
h(p)\coloneqq -\int_{\R^U} p(\psi)\log p(\psi)\,d\psi
\]
is the differential entropy. The associated variational problem is
\begin{equation}
\sup_{p\in\mathcal{A}(Q,\eps)} h(p).
\label{eq:master_variational_problem}
\end{equation}

This formulation clarifies the scope of the theory. The design class is determined by three ingredients only:
(i) the state space $\R^U$ of log-fields, (ii) the centering and quadratic-budget constraints, and (iii) the choice of local operator $Q$. The role of the operator is especially important: once $Q$ is fixed, the entropy maximizer---if it exists---must reveal the correlation geometry compatible with that operator.

In VKD language, this subsection isolates the law-design part of the mechanism. The deployment axis has already been fixed in Section~\ref{sec:method}; the remaining task is to derive the canonical latent law and the geometry it induces.

\subsection{Quadratic MaxEnt principle and operator-forced kernel geometry}
\label{sec:maxent_sketch}

The next theorem is the main design theorem for the solved quadratic VKD subfamily. It turns desiderata D1, D4, and D5 into a unique latent law, and it makes explicit how the operator budget fixes the scale of the optimizer. Relative to the earlier proof sketch, it yields the optimizer, its entropy value, the explicit scale, and an entropy-gap identity that certifies uniqueness.

\begin{theorem}[Design theorem for the quadratic VKD subfamily]
\label{prop:maxent_logfield}
Let $Q\succ 0$ be symmetric positive definite on $\R^U$, let $n=|U|$, and let $\eps>0$. Then the variational problem \eqref{eq:master_variational_problem} has a unique optimizer
\begin{equation}
p^\star_{Q,\eps}
=
\mathcal{N}\!\big(0,\Sigma_{Q,\eps}\big),
\qquad
\Sigma_{Q,\eps}=\frac{2\eps}{n}\,Q^{-1}.
\label{eq:general_maxent_optimizer}
\end{equation}
Equivalently,
\begin{equation}
p^\star_{Q,\eps}(\psi)
=
\frac{1}{(2\pi)^{n/2}\det(\Sigma_{Q,\eps})^{1/2}}
\exp\!\Big(-\frac12\psi^\top \Sigma_{Q,\eps}^{-1}\psi\Big),
\label{eq:general_gaussian_optimizer}
\end{equation}
with precision matrix
\[
\Sigma_{Q,\eps}^{-1}=\frac{n}{2\eps}\,Q.
\]
Moreover, for every $p\in\mathcal{A}(Q,\eps)$,
\begin{equation}
h(p^\star_{Q,\eps})-h(p)=\KL\!\big(p\,\|\,p^\star_{Q,\eps}\big)\ge 0,
\label{eq:entropy_gap_identity}
\end{equation}
so the optimizer is unique. Its entropy is
\begin{equation}
h(p^\star_{Q,\eps})
=
\frac12\log\!\Big((2\pi e)^n\det\!\big(\tfrac{2\eps}{n}Q^{-1}\big)\Big).
\label{eq:maxent_entropy_value}
\end{equation}
\end{theorem}

\begin{proof}[Proof sketch]
Let $p^\star$ denote the Gaussian density in \eqref{eq:general_maxent_optimizer}. Since
\[
\Sigma_{Q,\eps}^{-1}=\frac{n}{2\eps}Q,
\]
the quadratic constraint implies
\[
\E_{p^\star}\!\left[\frac12\langle \psi,Q\psi\rangle\right]
=
\frac12\tr(Q\Sigma_{Q,\eps})
=
\frac12\tr\!\Big(Q\frac{2\eps}{n}Q^{-1}\Big)
=
\eps,
\]
so $p^\star\in\mathcal{A}(Q,\eps)$. For any feasible $p$,
\[
\KL(p\|p^\star)
=
-h(p)-\int p(\psi)\log p^\star(\psi)\,d\psi.
\]
Because $\log p^\star(\psi)=c-\frac{n}{4\eps}\langle \psi,Q\psi\rangle$ for a constant $c$, and every feasible $p$ has the same normalization, mean, and energy budget, the second term depends only on $(Q,\eps)$ and coincides with $-h(p^\star)$. Hence \eqref{eq:entropy_gap_identity} holds. Uniqueness follows because $\KL(p\|p^\star)=0$ iff $p=p^\star$ a.e. The entropy formula is the standard entropy of a centered Gaussian with covariance $\Sigma_{Q,\eps}$.
\end{proof}

\begin{corollary}[Operator-forced geometry in the Dirichlet instantiation]
\label{cor:green_forcing}
Taking $Q=L_U$ in \Cref{prop:maxent_logfield} yields the unique entropy-maximizing log-field
\begin{equation}
\psi\sim \mathcal{N}\!\big(0,(\beta L_U)^{-1}\big),
\qquad
\beta=\frac{n}{2\eps}.
\label{eq:dirichlet_specialization}
\end{equation}
Its covariance is
\begin{equation}
\Cov(\psi)=\frac{2\eps}{n}\,L_U^{-1}=\beta^{-1}G_U,
\qquad
G_U\coloneqq L_U^{-1}.
\label{eq:green_forced_cov}
\end{equation}
Thus, within the local quadratic design class determined by the Dirichlet energy, the correlation geometry is the Dirichlet Green kernel.
\end{corollary}

\begin{remark}[What is and is not ``forced'' ]
The theorem does \emph{not} say that the Green kernel is universally optimal for every noise-design problem. It says something more precise: once the design class is fixed by a local quadratic budget with operator $Q$, the entropy maximizer has covariance proportional to $Q^{-1}$. The Green kernel is forced specifically because the operator chosen here is the Dirichlet Laplacian.
\end{remark}

\paragraph{Framework interpretation.}
\Cref{prop:maxent_logfield} solves the latent-law axis of the VKD mechanism-design layer, and \Cref{cor:green_forcing} shows that the second-order geometry is induced rather than tuned. What remains is the realization map $\ell$: how to turn the derived latent log-field into a positive, mean-preserving gate that can actually be deployed.

\subsection{From the MaxEnt log-field to the canonical realization map}
\label{sec:gff-to-chaos}

At this point the latent law and induced geometry have been derived. The remaining step in the design layer is the realization map $\ell$: how to turn the latent log-field into a positive gate satisfying D2--D3. The canonical answer in the solved VKD subfamily is the Wick-normalized exponential. Let
\begin{equation}
\psi\sim\mathcal{N}(0,C),
\qquad
C=(\beta L_U)^{-1}=\frac{2\eps}{n}G_U.
\label{eq:gff-standing}
\end{equation}
For a strength parameter $\gamma\in\R$, define the exact Wick-normalized exponential
\begin{equation}
\xi^{\mathrm{ex}}_\gamma(x)
\coloneqq
:\!\exp(\gamma\psi(x))\!:
=
\exp\!\Big(\gamma\psi(x)-\tfrac{\gamma^2}{2}C(x,x)\Big),
\qquad x\in U.
\label{eq:exact_gate}
\end{equation}
This is the canonical exact realization associated with the variationally derived log-field. Positivity comes from the exponential map; mean preservation comes from the Wick correction.

\begin{theorem}[Canonical realization in the solved VKD subfamily]
\label{thm:maxent-gff-chaos}
Under desiderata \textup{D1}--\textup{D5}, the canonical exact positive mean-one multiplicative gate is obtained by:
\begin{enumerate}[leftmargin=1.5em]
\item sampling the MaxEnt log-field
\[
\psi\sim\mathcal{N}\!\big(0,(\beta L_U)^{-1}\big),
\qquad
\beta=\frac{|U|}{2\eps},
\]
and
\item applying the Wick-normalized exponential \eqref{eq:exact_gate}.
\end{enumerate}
For any sites $x_1,\dots,x_m\in U$,
\begin{equation}
\E\!\left[\prod_{r=1}^{m}\xi^{\mathrm{ex}}_\gamma(x_r)\right]
=
\exp\!\left(
\gamma^2
\sum_{1\le a<b\le m}
C(x_a,x_b)
\right).
\label{eq:all_moments_gate}
\end{equation}
In particular,
\begin{align}
\E\big[\xi^{\mathrm{ex}}_\gamma(x)\big] &= 1,
\label{eq:chaos-mean1}\\
\E\big[\xi^{\mathrm{ex}}_\gamma(x)\xi^{\mathrm{ex}}_\gamma(y)\big]
&=
\exp\!\big(\gamma^2 C(x,y)\big).
\label{eq:chaos-second}
\end{align}
Hence the induced second-order gate kernel is
\begin{equation}
K_\gamma(x,y)
\coloneqq
\E\big[\xi^{\mathrm{ex}}_\gamma(x)\xi^{\mathrm{ex}}_\gamma(y)\big]
=
\exp\!\big(\gamma^2 C(x,y)\big).
\label{eq:induced_gate_kernel}
\end{equation}
\end{theorem}

\begin{proof}
Because $(\psi(x_1),\dots,\psi(x_m))$ is jointly Gaussian,
\[
\E\!\left[\exp\!\Big(\gamma\sum_{r=1}^m \psi(x_r)\Big)\right]
=
\exp\!\left(
\frac{\gamma^2}{2}
\sum_{a=1}^m\sum_{b=1}^m C(x_a,x_b)
\right).
\]
Multiplying by the Wick-normalization factor
\[
\exp\!\left(-\frac{\gamma^2}{2}\sum_{r=1}^m C(x_r,x_r)\right)
\]
leaves only the off-diagonal contribution, yielding \eqref{eq:all_moments_gate}. The one-point and two-point formulas are the cases $m=1$ and $m=2$.
\end{proof}

\begin{proposition}[Effective one-parameter scaling]
\label{prop:effective_tau}
Define
\begin{equation}
\tau\coloneqq \frac{\gamma^2}{\beta}=\frac{2\eps\,\gamma^2}{|U|}.
\label{eq:effective_tau}
\end{equation}
Then the exact gate law depends on $(\beta,\gamma)$ only through $\tau$. Equivalently, if
\[
Y\sim\mathcal{N}(0,\tau G_U),
\]
then
\begin{equation}
\xi^{\mathrm{ex}}_\gamma(x)
\ \stackrel{d}{=}\
\exp\!\Big(Y(x)-\tfrac12\Var(Y(x))\Big).
\label{eq:tau_reparam}
\end{equation}
In particular,
\begin{equation}
K_\gamma(x,y)=\exp\!\big(\tau\,G_U(x,y)\big).
\label{eq:kernel_tau}
\end{equation}
\end{proposition}

\begin{proof}
Since $\psi\sim \mathcal N(0,\beta^{-1}G_U)$, the rescaled field $Y\coloneqq \gamma\psi$ is Gaussian with covariance
\[
\Cov(Y)=\gamma^2\beta^{-1}G_U=\tau G_U.
\]
The exact gate is precisely the Wick exponential of $Y$, so its law is determined by the law of $Y$, hence by $\tau$ alone. Equation \eqref{eq:kernel_tau} follows from \eqref{eq:induced_gate_kernel}.
\end{proof}

\begin{proposition}[Small-strength expansion]
\label{prop:small_gamma}
For each fixed site $x\in U$,
\begin{equation}
\xi^{\mathrm{ex}}_\gamma(x)
=
1+\gamma\psi(x)+\frac{\gamma^2}{2}\Big(\psi(x)^2-C(x,x)\Big)+O_{L^2}(\gamma^3)
\qquad (\gamma\to 0).
\label{eq:small_gamma_expansion}
\end{equation}
Moreover, for any $x,y\in U$,
\begin{align}
\E\big[\xi^{\mathrm{ex}}_\gamma(x)\xi^{\mathrm{ex}}_\gamma(y)\big]
&=
1+\gamma^2 C(x,y)+O(\gamma^4),
\label{eq:kernel_small_gamma}\\
\Cov\big(\xi^{\mathrm{ex}}_\gamma(x),\xi^{\mathrm{ex}}_\gamma(y)\big)
&=
\gamma^2 C(x,y)+O(\gamma^4).
\label{eq:cov_small_gamma}
\end{align}
\end{proposition}

\begin{proof}
Expand
\[
\exp\!\Big(\gamma\psi(x)-\tfrac{\gamma^2}{2}C(x,x)\Big)
\]
in powers of $\gamma$ and collect terms up to order $\gamma^2$, which gives \eqref{eq:small_gamma_expansion}. Equation \eqref{eq:kernel_small_gamma} follows by expanding \eqref{eq:induced_gate_kernel}:
\[
\exp\!\big(\gamma^2 C(x,y)\big)=1+\gamma^2 C(x,y)+O(\gamma^4).
\]
Since $\E[\xi^{\mathrm{ex}}_\gamma(x)]=1$, subtracting one yields \eqref{eq:cov_small_gamma}.
\end{proof}

\paragraph{Interpretation of the small-$\gamma$ regime.}
\Cref{prop:small_gamma} shows that correlated additive Gaussian perturbation is only the \emph{first-order proxy} of the exact gate. The full multiplicative construction retains positivity, exact mean preservation, and higher-order lognormal structure. This is one mathematically precise sense in which \GCh{} is more than ``correlated Gaussian noise with a different parameterization.''

\paragraph{Exact theory versus implementation variant.}
The closed-form moment identities in \Cref{thm:maxent-gff-chaos,prop:effective_tau,prop:small_gamma} refer to the exact Wick-normalized gate \eqref{eq:exact_gate}. In practice, unless otherwise stated, our experiments use the sample-wise mean-one implementation variant
\begin{equation}
\xi^{\mathrm{sw}}_\gamma(x)
=
\frac{\exp(\gamma\psi(x))}{\frac{1}{|U|}\sum_{y\in U}\exp(\gamma\psi(y))},
\label{eq:samplewise_gate}
\end{equation}
which is also positive and satisfies
\begin{equation}
\frac{1}{|U|}\sum_{x\in U}\xi^{\mathrm{sw}}_\gamma(x)=1
\qquad\text{almost surely.}
\label{eq:samplewise_meanone}
\end{equation}
However, \eqref{eq:samplewise_gate} does \emph{not} preserve the exact sitewise moment formulas of the Wick-normalized gate in general. Accordingly, the exact gate is the canonical object in the theory, while the sample-wise variant is the optimization-friendly implementation used in the main experiments. Appendix~\ref{Norm} compares the two normalizations in more detail.

\paragraph{Framework interpretation.}
This separation is deliberate and belongs to the framework rather than only to this example. In VKD terms, $\xi^{\mathrm{ex}}$ is the canonical realization associated with the design desiderata, whereas $\xi^{\mathrm{sw}}$ is the deployed implementation used for training. The next subsection begins the compatibility layer for the latter.

\subsection{Compatibility layer for the implemented mechanism}
\label{sec:geometry_theory}

At this point the design layer is complete: the perturbation domain, deployment axis, latent law, induced geometry, and realization map have all been specified. We now turn to the second layer of VKD, namely compatibility on a target representation regime. In the language of Section~\ref{sec:framework}, the regime of interest here is the class of positive coherent late-semantic maps, and the main observables are pairwise log-ratio deformation, ranking stability, intrinsic roughness, and topology.

The exact Wick gate is the canonical object of the variational theory, but the main experiments use the sample-wise gate \eqref{eq:samplewise_gate}. That implementation admits a sharp geometric description because it differs from the unnormalized exponential by a \emph{spatially constant} correction in the log domain. This makes it possible to derive exact compatibility statements for the deployed mechanism rather than only for the canonical realization.

\paragraph{Framework reading guide.} The next results should be read as properties of the induced action of the implemented mechanism on the observables above. The pairwise log-ratio theorem gives the local relative-geometry law; the ranking corollary converts it into a probability of evidence preservation; the intrinsic-energy corollary gives the whole-map roughness budget; and the hard-mask results identify the failure mode of discontinuous deletion in the same representation regime. In one sentence: the implemented \GCh{} gate yields a \emph{finite, margin-aware Gaussian deformation} of relative geometry, whereas hard masking yields a \emph{singular, margin-blind deformation} whose relative damage worsens as representations become more coherent.

\paragraph{Assumption-to-practice map.} Three mathematical objects are worth translating immediately into deep-learning language. Positivity models post-ReLU, attention-weighted, or saliency-like activations that encode evidence by magnitude. Pairwise log-ratios model \emph{relative evidence}: how much more strongly one region or token is supported than another. The intrinsic energy $\mathcal E_{\mathrm{int}}$ measures edgewise variation after removing global scale, so low intrinsic energy corresponds to a map that is spatially coherent or low-frequency. With that translation in mind, the next results can be read as statements about evidence preservation, ranking preservation, and geometric distortion of semantic feature maps.

\paragraph{Positivity domain of the log-geometry statements.} All results in this subsection that involve $\log h$ or pairwise log-ratios are stated for strictly positive fields. This is deliberate: the mathematical object under study is the geometry of \emph{positive evidence maps}. In practice, the statements apply exactly on positive-support channels or regions, and one may also work with an $\varepsilon$-lifted field $h+\varepsilon$ if a numerical implementation needs to avoid exact zeros. The paper does not claim these log-geometry theorems for arbitrary signed activations.

\begin{theorem}[Pairwise log-ratio stability of the implemented \GCh{} gate]
\label{thm:samplewise_logratio}
Let $h:U\to(0,\infty)$ be a fixed positive field and define
\[
\widetilde h \coloneqq \xi^{\mathrm{sw}}_\gamma \odot h,
\]
where $\xi^{\mathrm{sw}}_\gamma$ is given by \eqref{eq:samplewise_gate}. For every $x,y\in U$,
\begin{equation}
\Delta^{\mathrm{sw}}_{xy}(h)
\coloneqq
\log\frac{\widetilde h(x)}{\widetilde h(y)}-\log\frac{h(x)}{h(y)}
=
\gamma\big(\psi(x)-\psi(y)\big).
\label{eq:samplewise_logratio}
\end{equation}
Consequently, $\Delta^{\mathrm{sw}}_{xy}(h)$ is centered Gaussian with variance
\begin{equation}
\Var\!\big(\Delta^{\mathrm{sw}}_{xy}(h)\big)
=
\gamma^2\big(C(x,x)+C(y,y)-2C(x,y)\big)
=
\tau R_G(x,y),
\label{eq:samplewise_logratio_var}
\end{equation}
where $\tau=\gamma^2/\beta$ and
\begin{equation}
R_G(x,y)\coloneqq G_U(x,x)+G_U(y,y)-2G_U(x,y).
\label{eq:green_metric}
\end{equation}
More generally, for any collection of pairs $\{(x_r,y_r)\}_{r=1}^m$, the vector
\[
\big(\Delta^{\mathrm{sw}}_{x_ry_r}(h)\big)_{r=1}^m
\]
is jointly Gaussian. In particular, $R_G(x,y)\ge 0$ for all $x,y\in U$ because it is the variance proxy of a Gaussian difference field.
\end{theorem}

\begin{proof}
Write
\[
c(\psi)\coloneqq \log\!\Big(\frac{1}{|U|}\sum_{z\in U} e^{\gamma\psi(z)}\Big).
\]
Then
\[
\log \widetilde h(x)=\log h(x)+\gamma\psi(x)-c(\psi)
\qquad\text{for every }x\in U.
\]
Subtracting the same identity at $y$ gives \eqref{eq:samplewise_logratio}. Since $\psi$ is Gaussian and $\Delta^{\mathrm{sw}}_{xy}(h)$ is a linear functional of $\psi$, it is centered Gaussian with variance
\[
\gamma^2\Var\!\big(\psi(x)-\psi(y)\big)
=
\gamma^2\big(C(x,x)+C(y,y)-2C(x,y)\big).
\]
Using $C=\beta^{-1}G_U$ yields the final expression $\tau R_G(x,y)$. Joint Gaussianity for finitely many pairs is immediate for the same reason.
\end{proof}

\begin{corollary}[Margin-sensitive ranking stability under the implemented \GCh{} gate]
\label{cor:ranking_stability}
Assume $x\neq y$, $h(x)>h(y)>0$, and $\tau R_G(x,y)>0$, and define the log-margin
\begin{equation}
\delta_{xy}(h)\coloneqq \log h(x)-\log h(y)>0.
\label{eq:log_margin}
\end{equation}
Then under the implemented sample-wise gate,
\begin{equation}
\Pr\big(\widetilde h(x)>\widetilde h(y)\big)
=
\Phi\!\left(\frac{\delta_{xy}(h)}{\sqrt{\tau R_G(x,y)}}\right),
\label{eq:ranking_stability_exact}
\end{equation}
where $\Phi$ is the standard Gaussian cdf. Equivalently,
\begin{equation}
\Pr\big(\widetilde h(x)\le \widetilde h(y)\big)
=
\Phi\!\left(-\frac{\delta_{xy}(h)}{\sqrt{\tau R_G(x,y)}}\right)
\le
\exp\!\left(-\frac{\delta_{xy}(h)^2}{2\tau R_G(x,y)}\right).
\label{eq:ranking_stability_tail}
\end{equation}
\end{corollary}

\begin{proof}
By \Cref{thm:samplewise_logratio},
\[
\log\frac{\widetilde h(x)}{\widetilde h(y)}
=
\log\frac{h(x)}{h(y)}+\Delta^{\mathrm{sw}}_{xy}(h)
=
\delta_{xy}(h)+\Delta^{\mathrm{sw}}_{xy}(h),
\]
where $\Delta^{\mathrm{sw}}_{xy}(h)\sim \mathcal N(0,\tau R_G(x,y))$. Therefore
\[
\Pr\big(\widetilde h(x)>\widetilde h(y)\big)
=
\Pr\big(\delta_{xy}(h)+\Delta^{\mathrm{sw}}_{xy}(h)>0\big)
=
\Phi\!\left(\frac{\delta_{xy}(h)}{\sqrt{\tau R_G(x,y)}}\right),
\]
which is \eqref{eq:ranking_stability_exact}. The tail bound follows from the standard Gaussian bound $\Phi(-u)\le e^{-u^2/2}$ for $u>0$.
\end{proof}

\paragraph{Deep learning interpretation.} Pairwise log-ratios are a natural coordinate system for relative evidence: how much stronger one region, token, or semantic part is than another. \Cref{cor:ranking_stability} says that the implemented \GCh{} gate is \emph{margin-sensitive}: if a feature comparison already has a large semantic log-margin, then the probability of preserving that ordering is exponentially close to one. This is the kind of behavior one wants from a late-stage regularizer---strong semantic contrasts become more, not less, stable. The mild condition $\tau R_G(x,y)>0$ simply excludes the degenerate zero-variance case; on a connected Dirichlet grid it is automatic whenever $x\neq y$ and $\gamma\neq 0$.

To aggregate pairwise distortions over the grid, define the \emph{intrinsic} interior edge set
\[
E_{\mathrm{int}}\coloneqq \{\{x,y\}\in E:\ x,y\in U\}
\]
and the associated intrinsic graph energy
\begin{equation}
\mathcal E_{\mathrm{int}}(f)
\coloneqq
\frac12\sum_{\{x,y\}\in E_{\mathrm{int}}} c_{xy}\big(f(x)-f(y)\big)^2
=
\frac12\langle f,L_{\mathrm{int}}f\rangle,
\label{eq:intrinsic_energy}
\end{equation}
where $L_{\mathrm{int}}$ is the interior graph Laplacian on $U$ with \emph{no} auxiliary boundary term. Unlike the Dirichlet energy used in the variational design, $\mathcal E_{\mathrm{int}}$ is invariant under adding spatial constants, so it measures \emph{relative} geometry.

\begin{corollary}[Exact expected intrinsic roughness budget under the implemented gate]
\label{cor:samplewise_intrinsic}
Let $h:U\to(0,\infty)$ and $\widetilde h=\xi^{\mathrm{sw}}_\gamma\odot h$. Then
\begin{equation}
\E\big[\mathcal E_{\mathrm{int}}(\log \widetilde h)\big]
=
\mathcal E_{\mathrm{int}}(\log h)
+
\gamma^2\eps_{\mathrm{int}},
\qquad
\eps_{\mathrm{int}}\coloneqq \E\big[\mathcal E_{\mathrm{int}}(\psi)\big]=\tfrac12\tr(L_{\mathrm{int}}C).
\label{eq:samplewise_intrinsic_budget}
\end{equation}
\end{corollary}

\begin{proof}
From the proof of \Cref{thm:samplewise_logratio},
\[
\log \widetilde h = \log h + \gamma\psi - c(\psi)\mathbf 1,
\]
where $\mathbf 1$ is the all-ones vector on $U$. Because $L_{\mathrm{int}}\mathbf 1=0$, the constant term drops out of $\mathcal E_{\mathrm{int}}$. Hence
\[
\mathcal E_{\mathrm{int}}(\log \widetilde h)
=
\mathcal E_{\mathrm{int}}(\log h+\gamma\psi).
\]
Expanding the quadratic form and taking expectation gives
\begin{align*}
\E\big[\mathcal E_{\mathrm{int}}(\log \widetilde h)\big]
&=
\mathcal E_{\mathrm{int}}(\log h)
+
\gamma\,\E\big[\langle \log h,L_{\mathrm{int}}\psi\rangle\big]
+
\gamma^2\E\big[\mathcal E_{\mathrm{int}}(\psi)\big] \\
&=
\mathcal E_{\mathrm{int}}(\log h)+\gamma^2\E\big[\mathcal E_{\mathrm{int}}(\psi)\big],
\end{align*}
where the cross term vanishes because $\E[\psi]=0$. Finally,
\[
\E\big[\mathcal E_{\mathrm{int}}(\psi)\big]
=
\frac12\E[\psi^\top L_{\mathrm{int}}\psi]
=
\frac12\tr(L_{\mathrm{int}}C).
\]
\end{proof}

\paragraph{Deep learning interpretation.} \Cref{cor:samplewise_intrinsic} is the whole-map counterpart of the pairwise result. In the intrinsic log-geometry of a positive feature map, the implemented \GCh{} gate adds an \emph{exactly quantified expected} amount of roughness. It deforms the representation by a finite random field rather than puncturing it with hard zeros. For practitioners, this is the rigorous version of the intuition that \GCh{} injects controlled uncertainty rather than discontinuous semantic damage.

\begin{corollary}[Scale compatibility of the implemented \GCh{} gate]
\label{cor:scale_compatibility}
For any $a>0$ and any positive field $h:U\to(0,\infty)$, let $\widetilde h_a\coloneqq \xi^{\mathrm{sw}}_\gamma\odot (a h)$. Then for every $x,y\in U$,
\begin{equation}
\Delta^{\mathrm{sw}}_{xy}(a h)=\Delta^{\mathrm{sw}}_{xy}(h),
\label{eq:scale_invariant_pairwise}
\end{equation}
and
\begin{equation}
\E\big[\mathcal E_{\mathrm{int}}(\log \widetilde h_a)\big]-\mathcal E_{\mathrm{int}}(\log(a h))
=
\gamma^2\eps_{\mathrm{int}}.
\label{eq:scale_invariant_intrinsic}
\end{equation}
Thus the pairwise deformation law and the added intrinsic roughness budget are invariant under global amplitude rescaling.
\end{corollary}

\begin{proof}
Because $\log(a h)=\log h+(\log a)\mathbf 1$, global rescaling adds only a spatial constant in the log domain. Both \Cref{thm:samplewise_logratio} and the intrinsic energy $\mathcal E_{\mathrm{int}}$ are invariant under such constants, which yields \eqref{eq:scale_invariant_pairwise} and \eqref{eq:scale_invariant_intrinsic}.
\end{proof}

\paragraph{Deep learning interpretation.} This is a concrete advantage of working in multiplicative log-geometry. If the same semantic feature map is globally rescaled---for example by a change in channel gain, normalization, or overall confidence level---the geometric effect of the implemented \GCh{} gate does not change. The perturbation tracks relative structure rather than absolute amplitude.

\begin{corollary}[Finite expected intrinsic roughness for a perfectly coherent positive map under the implemented gate]
\label{cor:gch_perfect_coherence}
Let $h:U\to(0,\infty)$ satisfy $\log h(x)\equiv c$ on $U$ for some constant $c\in\R$. Then for $\widetilde h=\xi^{\mathrm{sw}}_\gamma\odot h$,
\begin{equation}
\E\big[\mathcal E_{\mathrm{int}}(\log \widetilde h)\big]=\gamma^2\eps_{\mathrm{int}}.
\label{eq:gch_perfect_coherence}
\end{equation}
In particular, a perfectly coherent positive map acquires a finite and explicitly budgeted \emph{expected} intrinsic roughness under the implemented \GCh{} gate.
\end{corollary}

\begin{proof}
If $\log h$ is constant on $U$, then $\mathcal E_{\mathrm{int}}(\log h)=0$. The claim follows immediately from \Cref{cor:samplewise_intrinsic}.
\end{proof}

\paragraph{Deep learning interpretation.} A late-stage representation is often close to piecewise coherent in log-amplitude: within a semantically consistent region, the main issue is not whether the feature is exactly constant, but whether the perturbation preserves the region as a coherent object. \Cref{cor:gch_perfect_coherence} gives an expectation-level statement: starting from zero intrinsic roughness, the implemented \GCh{} gate produces a finite and explicitly budgeted expected roughness level rather than a singular or uncontrolled distortion.

The next result formalizes the opposite behavior of hard binary masks. The singular-ratio statement applies whenever a compared pair can be zeroed with positive probability, and therefore covers dropout, DropBlock, and related hard-masking mechanisms in their natural nontrivial regime.

\begin{theorem}[Binary masks are incompatible with finite log-ratio geometry]
\label{thm:binary_mask_singularity}
Let $h:U\to(0,\infty)$, let $a>0$, and let $m:U\to\{0,a\}$ be any random binary mask. Define $\widetilde h^{m}\coloneqq m\odot h$. If there exist $x,y\in U$ such that
\begin{equation}
\Pr\big(m(x)=0\ \text{or}\ m(y)=0\big)>0,
\label{eq:mask_singularity_condition}
\end{equation}
then
\[
\log\frac{\widetilde h^{m}(x)}{\widetilde h^{m}(y)}
\]
fails to be an almost surely finite real-valued random variable. In particular, no finite-variance analog of \Cref{thm:samplewise_logratio} can hold for such a mask. For inverted dropout at \emph{distinct} compared sites $x\neq y$,
\[
m_q(z)=\frac{b(z)}{q},\qquad b(z)\stackrel{\mathrm{i.i.d.}}{\sim}\mathrm{Bernoulli}(q),
\]
the total probability of a zero event at the compared pair is $1-q^2$, with asymmetric singular events of total probability $2q(1-q)$ and a joint-erasure event of probability $(1-q)^2$.
\end{theorem}

\begin{proof}
If $m(x)=0$ and $m(y)=a$, then $\widetilde h^m(x)=0<\widetilde h^m(y)$ and the log-ratio equals $-\infty$. If $m(x)=a$ and $m(y)=0$, then the log-ratio equals $+\infty$. If $m(x)=m(y)=0$, then both numerator and denominator vanish and the log-ratio is undefined, hence not a finite real number. Therefore the log-ratio fails to be almost surely finite whenever \eqref{eq:mask_singularity_condition} holds. For inverted dropout at distinct sites $x\neq y$, independence gives
\[
\Pr\big(m_q(x)=0\ \text{or}\ m_q(y)=0\big)=1-q^2,
\]
while the asymmetric events have total probability $2q(1-q)$ and the joint-erasure event has probability $(1-q)^2$.
\end{proof}

\begin{corollary}[Margin-blind ranking under inverted dropout]
\label{cor:dropout_margin_blind}
Assume $h(x)>h(y)>0$ and let $m_q$ be inverted dropout with keep probability $q\in(0,1]$. Then
\begin{equation}
\Pr\big((m_q\odot h)(x)>(m_q\odot h)(y)\big)=q.
\label{eq:dropout_margin_blind}
\end{equation}
In particular, the probability of preserving the ordering is independent of the magnitude of the underlying feature margin.
\end{corollary}

\begin{proof}
Write $m_q(z)=b(z)/q$ with $b(z)\in\{0,1\}$. If $b(x)=1$, then $(m_q\odot h)(x)=h(x)/q$ and regardless of whether $b(y)=0$ or $1$, one has $(m_q\odot h)(x)>(m_q\odot h)(y)$ because $h(x)>h(y)>0$. If $b(x)=0$, then $(m_q\odot h)(x)=0\le (m_q\odot h)(y)$. Therefore the ordering is preserved if and only if $b(x)=1$, which occurs with probability $q$.
\end{proof}

\paragraph{Deep learning interpretation.} This corollary is intentionally blunt: even if one activation is \emph{arbitrarily} more semantically decisive than another, inverted dropout preserves that ordering with probability exactly $q$ and destroys or erases it with probability $1-q$. In that sense hard masking is \emph{margin-blind}. By comparison, \Cref{cor:ranking_stability} shows that the implemented \GCh{} gate becomes more stable as the semantic margin increases.

\begin{proposition}[Exact intrinsic energy inflation under inverted dropout]
\label{prop:dropout_energy}
Let $m_q(x)=b(x)/q$ with i.i.d.\ $b(x)\sim\mathrm{Bernoulli}(q)$ and $q\in(0,1]$. Then for every deterministic field $h:U\to\R$,
\begin{equation}
\E\big[\mathcal E_{\mathrm{int}}(m_q\odot h)\big]
=
\mathcal E_{\mathrm{int}}(h)
+
\frac{1-q}{2q}\sum_{x\in U} d_x^{\mathrm{int}} h(x)^2,
\label{eq:dropout_energy_inflation}
\end{equation}
where
\[
d_x^{\mathrm{int}}\coloneqq \sum_{y:\{x,y\}\in E_{\mathrm{int}}} c_{xy}
\]
is the intrinsic weighted degree of $x$.
\end{proposition}

\begin{proof}
Fix an interior edge $\{x,y\}\in E_{\mathrm{int}}$. Since $m_q(x)$ and $m_q(y)$ are independent and $\E[m_q(x)]=1$, $\E[m_q(x)^2]=1/q$, we have
\begin{align*}
\E\big[(m_q(x)h(x)-m_q(y)h(y))^2\big]
&=
\frac{1}{q}h(x)^2+\frac{1}{q}h(y)^2-2h(x)h(y) \\
&=
\big(h(x)-h(y)\big)^2+\Big(\frac{1}{q}-1\Big)\big(h(x)^2+h(y)^2\big).
\end{align*}
Multiply by $c_{xy}/2$ and sum over $E_{\mathrm{int}}$. The first term sums to $\mathcal E_{\mathrm{int}}(h)$, while the second becomes
\[
\frac{1-q}{2q}\sum_{x\in U} d_x^{\mathrm{int}}h(x)^2.
\]
This is exactly \eqref{eq:dropout_energy_inflation}.
\end{proof}

\begin{corollary}[Coherence amplification factor for inverted dropout]
\label{cor:coherence_factor}
Assume $\mathcal E_{\mathrm{int}}(h)>0$ and define the coherence score
\begin{equation}
\kappa(h)\coloneqq \frac{\sum_{x\in U} d_x^{\mathrm{int}} h(x)^2}{2\mathcal E_{\mathrm{int}}(h)}.
\label{eq:coherence_score}
\end{equation}
Then inverted dropout satisfies
\begin{equation}
\frac{\E\big[\mathcal E_{\mathrm{int}}(m_q\odot h)\big]}{\mathcal E_{\mathrm{int}}(h)}
=
1+\frac{1-q}{q}\,\kappa(h).
\label{eq:coherence_factor}
\end{equation}
\end{corollary}

\begin{proof}
Divide both sides of \eqref{eq:dropout_energy_inflation} by $\mathcal E_{\mathrm{int}}(h)>0$ and rearrange.
\end{proof}

\paragraph{Deep learning interpretation.} The scalar $\kappa(h)$ is an interpretable mismatch factor: it is large when a feature map carries nontrivial activation mass but varies only weakly across space, i.e. when the representation is coherent. \Cref{cor:coherence_factor} therefore says that hard masking damages coherent representations more severely in relative terms, and it does so by a completely explicit amplification factor.

\begin{corollary}[Immediate loss of perfect coherence under inverted dropout in expectation]
\label{cor:dropout_perfect_coherence}
Assume $q\in(0,1)$ and that the interior graph has at least one edge. Let $h(x)\equiv c$ on $U$ for some constant $c\neq 0$. Then
\begin{equation}
\mathcal E_{\mathrm{int}}(h)=0,
\qquad
\E\big[\mathcal E_{\mathrm{int}}(m_q\odot h)\big]
=
\frac{1-q}{2q}c^2\sum_{x\in U} d_x^{\mathrm{int}}>0.
\label{eq:dropout_perfect_coherence}
\end{equation}
Thus perfect coherence is not preserved by a single masking step: the post-mask field has strictly positive \emph{expected} intrinsic roughness.
\end{corollary}

\begin{proof}
A constant field has zero intrinsic energy, so the claim follows immediately from \Cref{prop:dropout_energy}.
\end{proof}

\paragraph{Deep learning interpretation.} This is the cleanest possible statement of hard-mask mismatch. Even if a feature map is spatially perfectly coherent before perturbation, binary masking does not preserve that zero-roughness state in any controlled relative sense. After one masking step the representation acquires strictly positive \emph{expected} edgewise roughness, reflecting the discontinuities introduced by hard deletion.

\begin{corollary}[Late-stage mismatch of inverted dropout under coherence]
\label{cor:coherence_mismatch}
Let $(h_\ell)_{\ell\ge 1}$ be deterministic fields on $U$ such that
\begin{equation}
\inf_{\ell\ge 1}\sum_{x\in U} d_x^{\mathrm{int}} h_\ell(x)^2 >0,
\qquad
\mathcal E_{\mathrm{int}}(h_\ell)>0\ \text{for every }\ell,
\qquad
\mathcal E_{\mathrm{int}}(h_\ell)\to 0.
\label{eq:coherence_hypothesis}
\end{equation}
Then for every fixed $q\in(0,1)$,
\begin{equation}
\frac{\E\big[\mathcal E_{\mathrm{int}}(m_q\odot h_\ell)\big]-\mathcal E_{\mathrm{int}}(h_\ell)}{\mathcal E_{\mathrm{int}}(h_\ell)}
\longrightarrow \infty.
\label{eq:coherence_divergence}
\end{equation}
Thus, as the representation becomes more spatially coherent, the relative geometric distortion induced by binary masking diverges.
\end{corollary}

\begin{proof}
By \Cref{prop:dropout_energy},
\[
\frac{\E[\mathcal E_{\mathrm{int}}(m_q\odot h_\ell)]-\mathcal E_{\mathrm{int}}(h_\ell)}{\mathcal E_{\mathrm{int}}(h_\ell)}
=
\frac{1-q}{2q}\cdot \frac{\sum_{x\in U} d_x^{\mathrm{int}} h_\ell(x)^2}{\mathcal E_{\mathrm{int}}(h_\ell)}.
\]
The numerator is bounded below by assumption, whereas the denominator tends to zero, so the ratio diverges to $+\infty$.
\end{proof}

\begin{corollary}[Margin-growth regime: \GCh{} strengthens while dropout saturates]
\label{cor:margin_growth_regime}
Fix distinct $x,y\in U$ and assume $\tau R_G(x,y)>0$. Let $(h_\ell)_{\ell\ge 1}$ be positive fields with $h_\ell(x)>h_\ell(y)$ for every $\ell$. Define
\[
\delta_\ell\coloneqq \log h_\ell(x)-\log h_\ell(y).
\]
If
\begin{equation}
\frac{\delta_\ell}{\sqrt{\tau R_G(x,y)}}\longrightarrow \infty,
\label{eq:margin_growth_assumption}
\end{equation}
then under the implemented sample-wise \GCh{} gate,
\begin{equation}
\Pr\big(\widetilde h_\ell(x)>\widetilde h_\ell(y)\big)\longrightarrow 1.
\label{eq:margin_growth_gch}
\end{equation}
Under inverted dropout with keep probability $q$, however,
\begin{equation}
\Pr\big((m_q\odot h_\ell)(x)>(m_q\odot h_\ell)(y)\big)=q\qquad\text{for every }\ell.
\label{eq:margin_growth_dropout}
\end{equation}
\end{corollary}

\begin{proof}
Equation \eqref{eq:margin_growth_gch} follows immediately from \Cref{cor:ranking_stability} and the assumption \eqref{eq:margin_growth_assumption}. Equation \eqref{eq:margin_growth_dropout} is exactly \Cref{cor:dropout_margin_blind}.
\end{proof}

\paragraph{Deep learning interpretation.} \Cref{cor:margin_growth_regime} is the mathematically clean version of the informal slogan that \GCh{} becomes more compatible with later, sharper semantic representations. It does \emph{not} claim that depth is always beneficial in every model. Instead it says: whenever late-stage representations become more decisively separated in their relative log-margins, the implemented \GCh{} gate respects those rankings with probability tending to one, while hard masking stays stuck at the same keep-probability ceiling.

\begin{corollary}[Representation-compatibility dichotomy]
\label{cor:compatibility_dichotomy}
Under the hypotheses of \Cref{thm:samplewise_logratio,cor:samplewise_intrinsic,thm:binary_mask_singularity,cor:coherence_mismatch}, the implemented \GCh{} gate and hard binary masks exhibit qualitatively different behavior on positive coherent representations:
\begin{enumerate}[leftmargin=1.5em]
\item the implemented \GCh{} gate preserves a finite relative log-geometry, with exact Gaussian pairwise deformations, margin-sensitive ranking stability, and an exact additive intrinsic roughness budget;
\item any hard binary mask that can zero one or both members of a compared pair with positive probability fails to preserve finite log-ratio geometry, and inverted dropout preserves pairwise ranking only with the margin-blind probability $q$; and
\item for inverted dropout, the relative intrinsic distortion diverges along coherent representation sequences satisfying \eqref{eq:coherence_hypothesis}.
\end{enumerate}
\end{corollary}

The assumptions in \Cref{cor:coherence_mismatch} are a clean mathematical abstraction of the late-semantic regime: the representation retains nontrivial mass but becomes increasingly low-frequency or spatially coherent. In that regime, binary masking becomes more and more mismatched. By contrast, \Cref{thm:samplewise_logratio,cor:ranking_stability,cor:samplewise_intrinsic,cor:margin_growth_regime,cor:compatibility_dichotomy} show that the implemented \GCh{} gate continues to produce a finite Gaussian deformation whose pairwise, ranking, and aggregate effects are controlled by the Green geometry.

\paragraph{Engineering takeaway.} If a layer encodes positive region-level evidence or token-level saliency, then the mathematically relevant question is not merely whether noise is mean-preserving in expectation, but whether it preserves \emph{relative comparisons} that the downstream model relies on. The results above say that \GCh{} perturbs those comparisons through a finite, margin-aware Gaussian deformation, whereas hard binary masks can delete them outright and become especially mismatched when the representation is coherent and semantically sharp.

A topological complement is given in Appendix~\ref{app:topology}: positive multiplicative gates perturb superlevel sets only through a multiplicative threshold band, whereas hard Bernoulli masking destroys loop-type excursion topology with probability $1-q^n$ on an $n$-cycle.

\paragraph{What these theorems do and do not claim.} They do \emph{not} prove that one should always inject noise deeper, nor that every masking strategy is inferior in every possible regime. What they prove is a sharper and more defensible statement: once a layer behaves like a positive coherent evidence map, there is a mathematically meaningful comparison to make. In that regime, the implemented \GCh{} gate preserves finite relative geometry, ranking information, and an explicit global roughness budget, while hard binary masking either makes those quantities singular or amplifies their distortion by an explicit coherence factor. That is exactly the regime targeted by the late-stage experiments in this paper.

\subsection{Implementation and efficient sampling}
\label{sec:implementation-interface}

\paragraph{Injecting the gate.}
Given a feature map $F\in\R^{C\times H\times W}$, we inject the spatial gate multiplicatively:
\begin{equation}
\widetilde{F}_c(x)=F_c(x)\,\xi_\gamma(x),
\qquad x\in U.
\label{eq:inject_gate}
\end{equation}
In the experiments, $\beta$ is fixed once the grid, operator, and normalization convention are chosen; $\gamma$ is the reported strength knob.

\paragraph{FFT/DST sampling of the GFF log-field.}
For the unweighted four-neighbor Dirichlet Laplacian on the $H\times W$ interior grid $U$, the eigenbasis is the 2D sine basis:
\begin{align}
e_{k,\ell}(i,j)
&=
\sin\!\Big(\frac{\pi k i}{H+1}\Big)\,
\sin\!\Big(\frac{\pi \ell j}{W+1}\Big),
\\
\lambda_{k,\ell}
&=
4\sin^2\!\Big(\frac{\pi k}{2(H+1)}\Big)
+
4\sin^2\!\Big(\frac{\pi \ell}{2(W+1)}\Big),
\label{eq:dirichlet-eigs}
\end{align}
for $1\le k\le H$ and $1\le \ell\le W$. Hence sampling
\[
\psi\sim\mathcal{N}(0,(\beta L_U)^{-1})
\]
reduces to spectral synthesis: draw i.i.d.\ $Z_{k,\ell}\sim\mathcal{N}(0,1)$, set
\[
A_{k,\ell}=\frac{Z_{k,\ell}}{\sqrt{\beta\,\lambda_{k,\ell}}},
\]
and compute $\psi=\mathrm{IDST2}(A)$ using an orthonormal inverse discrete sine transform. Fast DST implementations rely on FFT internally, giving near-linear complexity in the number of spatial sites.

\FloatBarrier
\begin{algorithm}[ht]
\caption{\GCh{} on an $H\times W$ grid (Dirichlet; FFT/DST implementation)}
\label{alg:gch}
\begin{algorithmic}[1]
\State \textbf{Input:} grid size $(H,W)$, parameters $\beta>0$, $\gamma\in\R$, feature map $F\in\R^{C\times H\times W}$
\State \textbf{Precompute once:} eigenvalues $\lambda_{k,\ell}$ in \eqref{eq:dirichlet-eigs}; choose a DST convention; optionally precompute the variance map $v(x)=C(x,x)$
\State \textbf{Sample spectral coefficients:} draw i.i.d.\ $Z_{k,\ell}\sim\mathcal{N}(0,1)$
\State \textbf{Scale by the Laplacian spectrum:} set $A_{k,\ell}\leftarrow Z_{k,\ell}/\sqrt{\beta\,\lambda_{k,\ell}}$
\State \textbf{Inverse transform:} $\psi\leftarrow \mathrm{IDST2}(A)$ \hfill (so $\psi\sim\mathcal{N}(0,(\beta L_U)^{-1})$)
\State \textbf{Exponentiate:} $G(x)\leftarrow \exp(\gamma\psi(x))$ for all $x\in U$
\State \textbf{Normalize (choose one):}
\State \quad \textbf{Exact Wick:} $\xi(x)\leftarrow \exp(\gamma\psi(x)-\tfrac{\gamma^2}{2}v(x))$
\State \quad \textbf{Sample-wise mean-one:} $\xi(x)\leftarrow G(x)\Big/\Big(\frac{1}{|U|}\sum_{y\in U}G(y)\Big)$
\State \textbf{Inject into features:} $\widetilde{F}_c(x)\leftarrow F_c(x)\,\xi(x)$ for all channels $c$ and sites $x\in U$
\State \textbf{Output:} noised feature map $\widetilde{F}$
\end{algorithmic}
\end{algorithm}

\section{Experiments}
\label{sec:experiments}

We evaluate whether the design principles behind \GCh{} translate into practical gains. Our empirical questions are: (i) which ingredients matter beyond raw noise magnitude, (ii) where in network depth is the mechanism most effective, and (iii) whether the effect transfers beyond the primary CNN setting. Detailed protocols are provided in Appendix~\ref{app:experiments}.

\paragraph{Theory-to-experiment map.} The representation-compatibility results make four concrete empirical predictions. Pairwise log-ratio stability and the margin-sensitive ranking law predict that when late-stage representations encode decisive relative evidence, \GCh{} should preserve that evidence better than hard masking. The intrinsic roughness budget predicts a broad non-destructive regime of stochasticity rather than abrupt fragmentation. The immediate loss-of-perfect-coherence and coherence-mismatch results predict that once a representation becomes spatially coherent, binary masking should incur disproportionate damage, especially at late stages. Finally, the topological appendix is most relevant to the fine-grained Pets pilot, where preserving coherent part structure matters most directly.

\subsection{Controlled multi-baseline ImageNet study}
\label{sec:main_imagenet}

To isolate the source of the gains, we run a controlled 3-seed comparison that separates noise magnitude, spatial correlation, and positivity/mean-one multiplicative gating. We compare \GCh{} against Dropout and DropBlock, together with additive Gaussian baselines (i.i.d.\ and correlated) whose injected strength is energy-matched. Table~\ref{tab:imagenet_clean_main} reports mean$\pm$std.

Unless otherwise stated, the \GCh{} experiments use the sample-wise mean-one normalization of Algorithm~\ref{alg:gch}, which is the implementation variant used throughout the main body.

\paragraph{Unified strength knob.}
To keep tables compact, we write $g$ for the method-specific strength parameter. For \GCh{}, $g\equiv \gamma$; for Gaussian baselines, $g\equiv \sigma$; for Dropout and DropBlock, $g\equiv p$; and for the no-noise baseline, $g=0$.

\begin{table}[ht]
\centering
\small
\setlength{\tabcolsep}{4pt} 
\renewcommand{\arraystretch}{1.05} 

\begin{tabular}{lcccc}
\toprule
Method & $g$ & Top-1$\uparrow$ & NLL$\downarrow$ & ECE$\downarrow$ \\
\midrule
None & 0 & 0.765$\pm$0.001 & 0.931$\pm$0.004 & 0.030$\pm$0.001 \\
Dropout & 0.1 & 0.764$\pm$0.001 & 0.942$\pm$0.005 & 0.033$\pm$0.001 \\
DropBlock & 0.1 & 0.765$\pm$0.000 & 0.930$\pm$0.002 & 0.032$\pm$0.000 \\
IID Gauss. & 0.1 & 0.765$\pm$0.001 & 0.930$\pm$0.005 & 0.032$\pm$0.002 \\
Cor. Gauss. & 0.1 & 0.765$\pm$0.000 & 0.944$\pm$0.002 & 0.037$\pm$0.001 \\
\textbf{GCh (ours)} & 0.1 & 0.764$\pm$0.001 & 0.934$\pm$0.004 & \textbf{0.020}$\pm$0.001 \\
\bottomrule
\end{tabular}
\caption{ImageNet val (uncorrupted) under late-stage injection (layer4). Mean$\pm$std over 3 seeds. Here $g$ denotes the method-specific strength knob: $g=\gamma$ for GCh, $g=\sigma$ for Gaussian baselines, and $g=p$ for Dropout/DropBlock.}
\label{tab:imagenet_clean_main}
\end{table}

\FloatBarrier

\subsection{Transfer to a transformer backbone: Swin-T}
\label{sec:swin_main}

We additionally evaluate \GCh{} on Swin-T under the same full-recipe training setup. Direct Dropout/DropBlock analogues are less aligned with transformer pipelines because token-based representations and attention updates no longer correspond to contiguous suppression on a convolutional feature grid, and standard transformer regularization usually acts on different objects (e.g., stochastic depth, attention dropout, or MLP dropout). We therefore report the clean full-recipe baseline and isolate the incremental effect of \GCh{} in this setting.

\Cref{tab:swint_fullrecipe} reports best-checkpoint performance.

\begin{table}[ht]
\centering
\small
\setlength{\tabcolsep}{6pt}
\begin{tabular}{lccc}
\toprule
Method & Top-1 Acc. $\uparrow$ & NLL $\downarrow$ & ECE $\downarrow$ \\
\midrule
Baseline (None) & 80.03\% & 0.9213 & 0.0762 \\
GCh\ (ours)    & \textbf{80.11\%} & \textbf{0.9131} & \textbf{0.0738} \\
\bottomrule
\end{tabular}
\caption{\textbf{Swin-T (best checkpoint).} Full-recipe training, single run.}
\label{tab:swint_fullrecipe}
\end{table}
\FloatBarrier

\subsection{What the evidence shows}
\label{sec:analysis_all}

The ImageNet controlled study supports three main conclusions.

First, \emph{correlation alone is not enough}. In Table~\ref{tab:imagenet_clean_main}, the correlated additive Gaussian baseline can worsen calibration relative to the no-noise baseline at matched strength. The strongest ECE improvements appear only when correlation is combined with \emph{positive mean-one multiplicative gating}, namely in \GCh{}.

Second, \emph{depth matters}. Injection depth induces a clear accuracy--calibration trade-off (Appendix Table~\ref{tab:depth_ablation}): moving from earlier to later stages substantially improves calibration while changing accuracy only modestly. In the selected 7-corruption ImageNet-C evaluation, the late-stage setting reduces ECE by 46\% and improves NLL by 3.3\% relative to the no-noise baseline (Appendix Table~\ref{tab:imagenetc_overall}), with corruption-wise details in Table~\ref{tab:imagenetc_breakdown}.

Third, \emph{there is a stable operating regime}. A strength sweep reveals a broad useful range around $g\approx 0.07$--$0.18$. When $g$ becomes too large, accuracy collapses and NLL rises sharply; ECE can also become misleadingly small under severe underconfidence, which we treat as a failure mode rather than a favorable outcome (Appendix Table~\ref{tab:imagenetc_gamma}). This empirical pattern is consistent with \Cref{thm:samplewise_logratio,cor:ranking_stability,cor:samplewise_intrinsic,cor:gch_perfect_coherence,cor:dropout_margin_blind,cor:dropout_perfect_coherence,cor:coherence_mismatch}: the implemented \GCh{} gate induces a finite Gaussian deformation in relative log-geometry together with an exact expected intrinsic roughness budget, while also becoming more stable when semantic margins are sharper; hard masking, by contrast, is margin-blind and incurs a coherence-sensitive geometric penalty whose relative size grows in the coherent-representation regime.

Finally, the Swin-T result provides preliminary transfer evidence beyond the primary ResNet-50 setting, and the Oxford-IIIT Pets pilot in the appendix supports the claim that spatially coherent positive gating is especially compatible with fine-grained structure-sensitive recognition.

\section{Conclusion}
\label{sec:conclusion}

We proposed Variational Kernel Design as a compositional framework for internal noise in deep learning. In VKD, a stochastic mechanism is not chosen from a fixed heuristic menu; it is derived from learning-relevant constraints and then analyzed on the representation regime where it is actually deployed. This two-layer viewpoint---mechanism design followed by compatibility analysis---is the main conceptual contribution of the paper.

Within the solved quadratic VKD subfamily studied here, we formulated the log-field construction problem as an explicit finite-dimensional maximum-entropy program under a quadratic operator budget, solved that program in closed form, and obtained an entropy-gap identity showing that the optimizer is uniquely Gaussian. For the Dirichlet operator, this makes the Green kernel emerge as the induced correlation geometry; after Wick normalization, it yields the canonical exact \GCh{} gate. Once the operator and energy budget are fixed, the exact gate becomes an effectively one-parameter family through $\tau=\gamma^2/\beta$.

The more distinctive message of the paper is the representation-compatibility layer that sits on top of this variational design. For the sample-wise gate actually used in the experiments, we established exact Gaussian control of pairwise log-ratios, margin-sensitive ranking stability, and an exact expected intrinsic roughness budget. For hard binary masking, we proved the opposite kind of statement: incompatibility with finite log-ratio geometry, margin-blind ranking under inverted dropout, immediate loss of perfect coherence in expectation on perfectly coherent maps, and a relative distortion term that diverges in the coherent-representation regime. The central contrast is therefore not merely \emph{Gaussian versus Bernoulli}; it is \emph{finite, margin-aware deformation versus singular or coherence-amplified deletion}.

These theorems are intentionally conditional rather than universal. They do not claim that every deeper layer in every architecture will automatically favor \GCh{}. They claim something more precise and, for practice, more useful: whenever positive semantic representations become coherent and their relative evidence sharpens, smooth multiplicative gating preserves those comparisons in a way that hard deletion cannot. That conditional form is exactly what allows the theory to speak directly to the late-stage regime without pretending to replace empirical evaluation.

Empirically, \GCh{} improves calibration on clean ImageNet, improves both ECE and NLL on a selected 7-corruption ImageNet-C evaluation, remains effective in late semantic stages where hard masking can degrade clean calibration, and shows encouraging transfer to Swin-T and a fine-grained pilot. The practical takeaway is simple: if a layer carries positive, coherent, region-level evidence, then the right question is not merely whether noise is unbiased, but whether it perturbs relative evidence smoothly or deletes it abruptly, and whether that perturbation respects the comparisons the downstream network actually uses. Our theory says that \GCh{} does the former, whereas canonical hard binary masks such as dropout and DropBlock-type deletion mechanisms tend to do the latter in the coherent late-stage regime.

More broadly, the paper suggests a reusable recipe for future work. First choose the operator that encodes the geometry one wants the noise to respect. Then derive the corresponding latent law and realization. Finally, ask whether the implemented mechanism is compatible with the representation regime of interest. That perspective opens the door to principled variants based on massive, anisotropic, graph-adapted, or architecture-specific operators while preserving the same mathematical blueprint.

\bibliography{references_final}

\appendix

\section{Notation and Terminology (Glossary)}
\label{app:glossary}

\begin{itemize}[leftmargin=1.5em]
\item $U$: the $H\times W$ feature grid on which the gate is sampled and applied.
\item $B$: the auxiliary Dirichlet boundary outside $U$; $\bar U=U\cup B$.
\item $L_U$: Dirichlet Laplacian on $U$; $G_U=L_U^{-1}$: Dirichlet Green kernel.
\item $\mathcal F$: law family of latent log-fields in the VKD mechanism.
\item $K$: intended second-order geometry / kernel in the VKD mechanism.
\item $\mathcal T$: injection operator in the VKD mechanism.
\item $\ell$: realization map from latent log-field to positive gate.
\item $\mathcal R$: target representation regime for compatibility analysis.
\item $\psi$: log-field; $\xi$: positive multiplicative gate; $\gamma$: \GCh{} strength parameter.
\item $g$: unified strength knob in the experimental tables ($g=\gamma/\sigma/p$ depending on the method).
\item \textbf{\GCh}: Gaussian Chaos Noise / gate (ours).
\item \textbf{IID/Corr. Gaussian}: additive Gaussian baselines with matched injected energy.
\end{itemize}

\section{Full variational derivation for Theorem~\ref{prop:maxent_logfield}}
\label{app:variational_derivation}

This appendix gives a fuller proof of the quadratic MaxEnt principle, including an entropy-gap identity and, for completeness, the corresponding Euler--Lagrange stationarity calculation.

Let $Q\succ 0$ be symmetric positive definite on $\R^U$, let $n=|U|$, and let $\eps>0$. Recall the admissible class
\[
\mathcal{A}(Q,\eps)
=
\left\{
p:\R^U\to[0,\infty)\ :\
\int p=1,\ 
\int \psi\,p(\psi)\,d\psi=0,\ 
\int \frac12\langle \psi,Q\psi\rangle p(\psi)\,d\psi=\eps,\ 
h(p)>-\infty
\right\}.
\]

\subsection{Entropy-gap proof of optimality and uniqueness}

Define
\[
\Sigma_{Q,\eps}\coloneqq \frac{2\eps}{n}Q^{-1},
\qquad
p^\star(\psi)\coloneqq
\frac{1}{(2\pi)^{n/2}\det(\Sigma_{Q,\eps})^{1/2}}
\exp\!\Big(-\frac12\psi^\top\Sigma_{Q,\eps}^{-1}\psi\Big).
\]
Since $\Sigma_{Q,\eps}^{-1}=\frac{n}{2\eps}Q$, we have
\[
\E_{p^\star}\!\left[\frac12\langle \psi,Q\psi\rangle\right]
=
\frac12\tr(Q\Sigma_{Q,\eps})
=
\frac12\tr\!\Big(Q\frac{2\eps}{n}Q^{-1}\Big)
=
\eps,
\]
and clearly $\E_{p^\star}[\psi]=0$, so $p^\star\in\mathcal A(Q,\eps)$.

Now fix any $p\in\mathcal A(Q,\eps)$. Using the definition of KL divergence,
\begin{align}
\KL(p\|p^\star)
&=
\int p(\psi)\log\frac{p(\psi)}{p^\star(\psi)}\,d\psi \nonumber\\
&=
-h(p)-\int p(\psi)\log p^\star(\psi)\,d\psi.
\label{eq:kl_gap_step1}
\end{align}
Since
\[
\log p^\star(\psi)
=
-\frac{n}{2}\log(2\pi)
-\frac12\log\det(\Sigma_{Q,\eps})
-\frac12\psi^\top\Sigma_{Q,\eps}^{-1}\psi,
\]
and $\Sigma_{Q,\eps}^{-1}=\frac{n}{2\eps}Q$, the energy constraint gives
\begin{align}
-\int p(\psi)\log p^\star(\psi)\,d\psi
&=
\frac{n}{2}\log(2\pi)+\frac12\log\det(\Sigma_{Q,\eps})
+\frac{n}{2\eps}\int \frac12\langle \psi,Q\psi\rangle p(\psi)\,d\psi \nonumber\\
&=
\frac{n}{2}\log(2\pi)+\frac12\log\det(\Sigma_{Q,\eps})+\frac{n}{2}.
\label{eq:kl_gap_step2}
\end{align}
But the right-hand side is exactly the entropy of $p^\star$:
\[
h(p^\star)
=
\frac12\log\!\Big((2\pi e)^n\det(\Sigma_{Q,\eps})\Big).
\]
Therefore \eqref{eq:kl_gap_step1} and \eqref{eq:kl_gap_step2} imply
\[
\KL(p\|p^\star)=h(p^\star)-h(p).
\]
Because $\KL(p\|p^\star)\ge 0$, we obtain
\[
h(p)\le h(p^\star),
\]
with equality iff $\KL(p\|p^\star)=0$, i.e.\ iff $p=p^\star$ almost everywhere. This proves both optimality and uniqueness.

\subsection{Euler--Lagrange derivation (for completeness)}

The same optimizer can be recovered by stationarity. Introduce Lagrange multipliers $\lambda_0\in\R$, $\lambda\in\R^U$, and $\beta\in\R$, and define
\begin{align}
\mathcal{L}(p)
&=
-\int p(\psi)\log p(\psi)\,d\psi
+\lambda_0\!\left(\int p(\psi)\,d\psi-1\right)\nonumber\\
&\quad
+\left\langle \lambda,\int \psi\,p(\psi)\,d\psi\right\rangle
-\beta\!\left(\int \frac12\langle \psi,Q\psi\rangle p(\psi)\,d\psi-\eps\right).
\label{eq:EL_lagrangian}
\end{align}
For an interior optimum, the first variation in a direction $\delta p$ gives
\[
\int \Big(
-\log p(\psi)-1+\lambda_0+\langle \lambda,\psi\rangle-\beta\tfrac12\langle \psi,Q\psi\rangle
\Big)\delta p(\psi)\,d\psi=0.
\]
Hence the Euler--Lagrange equation is
\[
-\log p(\psi)-1+\lambda_0+\langle \lambda,\psi\rangle-\beta\tfrac12\langle \psi,Q\psi\rangle=0,
\]
so
\[
p(\psi)\propto \exp(\langle \lambda,\psi\rangle)\,
\exp\!\Big(-\beta\tfrac12\langle \psi,Q\psi\rangle\Big).
\]
The centering constraint forces $\lambda=0$, and integrability requires $\beta>0$ because $Q\succ 0$. Thus
\[
p(\psi)\propto \exp\!\Big(-\beta\tfrac12\langle \psi,Q\psi\rangle\Big)
=
\exp\!\Big(-\tfrac12\psi^\top(\beta Q)\psi\Big),
\]
which is the centered Gaussian $\mathcal N(0,(\beta Q)^{-1})$. Matching the energy budget yields
\[
\eps
=
\frac12\tr\!\big(Q(\beta Q)^{-1}\big)
=
\frac{n}{2\beta},
\qquad\text{so}\qquad
\beta=\frac{n}{2\eps}.
\]
This reproduces
\[
(\beta Q)^{-1}=\frac{2\eps}{n}Q^{-1}=\Sigma_{Q,\eps}.
\]

\subsection{Dirichlet specialization}

Setting $Q=L_U$ gives the optimizer used in the main text:
\[
p^\star_{L_U,\eps}
=
\mathcal N\!\big(0,(\beta L_U)^{-1}\big),
\qquad
\beta=\frac{|U|}{2\eps}.
\]
Its covariance is
\[
\Cov(\psi)=\frac{2\eps}{|U|}L_U^{-1}=\beta^{-1}G_U.
\]

\paragraph{Other boundary conditions.}
If one uses periodic or Neumann boundary conditions on a connected finite graph, the Laplacian has a constant nullspace, so the corresponding field must be defined after gauge fixing, for example by pinning one site or imposing zero spatial mean and using the Moore--Penrose pseudoinverse. Under the auxiliary Dirichlet boundary used in the main text, $L_U\succ 0$ and no additional gauge fixing is needed.

\paragraph{Massive variant.}
A regularized or \emph{massive} variant replaces $L_U$ by $L_U+\mu I$ for $\mu>0$:
\[
\psi\sim\mathcal{N}\!\big(0,(\beta(L_U+\mu I))^{-1}\big).
\]
This corresponds to the quadratic energy
\[
\mathcal{E}_\mu(\psi)=\frac12\psi^\top(L_U+\mu I)\psi
\]
and yields a better conditioned covariance with shorter-range correlations.

\section{Further properties of the exact Gaussian-chaos gate}
\label{app:exact_gate_props}

This appendix collects simple but useful consequences of the exact construction.

\subsection{All-order moment formula}

Let $\xi^{\mathrm{ex}}_\gamma$ be defined by \eqref{eq:exact_gate}. For any $x_1,\dots,x_m\in U$,
\[
\E\!\left[\prod_{r=1}^{m}\xi^{\mathrm{ex}}_\gamma(x_r)\right]
=
\exp\!\left(
\gamma^2\sum_{1\le a<b\le m} C(x_a,x_b)
\right).
\]
The proof is the same Gaussian moment-generating calculation used in \Cref{thm:maxent-gff-chaos}.

\subsection{Effective one-parameter scaling}

Writing $\tau=\gamma^2/\beta$, the exact gate can be rewritten as the Wick exponential of a Gaussian field $Y\sim\mathcal N(0,\tau G_U)$:
\[
\xi^{\mathrm{ex}}_\gamma(x)
\stackrel{d}{=}
\exp\!\Big(Y(x)-\tfrac12\Var(Y(x))\Big).
\]
Hence the exact gate law depends on $(\gamma,\beta)$ only through $\tau$. This is the precise sense in which, once the operator and energy budget are fixed, the exact mechanism becomes an effectively one-parameter family.

\subsection{Small-strength regime}

Expanding the exact gate at small $\gamma$ gives
\[
\xi^{\mathrm{ex}}_\gamma(x)
=
1+\gamma\psi(x)+\frac{\gamma^2}{2}\big(\psi(x)^2-C(x,x)\big)+O_{L^2}(\gamma^3).
\]
Consequently,
\[
\E\big[\xi^{\mathrm{ex}}_\gamma(x)\xi^{\mathrm{ex}}_\gamma(y)\big]
=
1+\gamma^2C(x,y)+O(\gamma^4),
\qquad
\Cov\big(\xi^{\mathrm{ex}}_\gamma(x),\xi^{\mathrm{ex}}_\gamma(y)\big)
=
\gamma^2C(x,y)+O(\gamma^4).
\]
This makes explicit that additive correlated Gaussian noise is a first-order approximation of the exact gate but does not preserve the positivity or higher-order structure of the multiplicative mechanism.

\section{Why the Green kernel emerges in this design class}
\label{app:green_kernel}

The purpose of this appendix is to state the precise structural lesson of the variational analysis. The Green kernel is not an extra hypothesis layered on top of the model. It appears because the design class is built from a \emph{quadratic operator budget}, and the entropy maximizer in such a class always has covariance equal to the inverse of that operator.

\paragraph{General operator principle.}
Let $Q\succ 0$ be any symmetric positive definite operator on $\R^U$ and consider the admissible class $\mathcal A(Q,\eps)$ from \eqref{eq:admissible_class}. By \Cref{prop:maxent_logfield}, the unique entropy maximizer is
\[
\mathcal N\!\big(0,\tfrac{2\eps}{|U|}Q^{-1}\big).
\]
Hence the covariance kernel is forced to be proportional to $Q^{-1}$.

\paragraph{Dirichlet specialization.}
In the main text, the operator in the budget is the Dirichlet Laplacian $Q=L_U$, so the covariance becomes
\[
\Cov(\psi)=\frac{2\eps}{|U|}L_U^{-1}=\beta^{-1}G_U.
\]
This is the exact sense in which the Dirichlet Green kernel is \emph{forced}: it is the inverse operator corresponding to the chosen local smoothness budget.

\paragraph{Operator substitution principle.}
The same reasoning immediately yields a family of designed noises.
\begin{corollary}[Replacing the operator replaces the kernel]
\label{cor:operator_substitution}
Fix any SPD operator $Q$ on $\R^U$ and define the quadratic budget
\[
\E\!\left[\frac12\langle \psi,Q\psi\rangle\right]=\eps.
\]
Then the unique maximum-entropy log-field is Gaussian with covariance
\[
\Cov(\psi)=\frac{2\eps}{|U|}Q^{-1}.
\]
After Wick normalization, the exact multiplicative gate has kernel
\[
K_\gamma(x,y)=\exp\!\Big(\gamma^2 \frac{2\eps}{|U|}Q^{-1}(x,y)\Big).
\]
\end{corollary}

This corollary is useful conceptually. It shows that VKD is not tied to one operator or one architecture. Choosing $Q=L_U$ gives the massless Dirichlet construction of the main paper; choosing $Q=L_U+\mu I$ gives a massive variant with shorter-range correlations; choosing an anisotropic or graph-adapted operator would produce the corresponding inverse-kernel geometry. The core variational logic remains unchanged.

\section{Topological stability of positive gates and fracture under hard masks}
\label{app:topology}

For a positive field $f:U\to(0,\infty)$ and a threshold $t>0$, define the superlevel set
\[
S_t(f)\coloneqq \{x\in U: f(x)\ge t\},
\]
and view it as the induced subgraph of the underlying adjacency graph on $U$.

\begin{proposition}[Threshold-band stability under positive multiplicative gating]
\label{prop:superlevel_sandwich}
Let $h:U\to(0,\infty)$, let $\xi:U\to(0,\infty)$, and assume $\|\log \xi\|_\infty\le \eta$. Then for every $t>0$,
\begin{equation}
S_{t e^{\eta}}(h)\subseteq S_t(\xi\odot h)\subseteq S_{t e^{-\eta}}(h).
\label{eq:superlevel_sandwich}
\end{equation}
In particular, the superlevel topology of $\xi\odot h$ at level $t$ can differ from that of $h$ only through threshold events already present in the band $[t e^{-\eta},t e^{\eta}]$.
\end{proposition}

\paragraph{Deep learning interpretation.} The proposition says that a positive multiplicative gate does not tear the superlevel geometry apart arbitrarily. It only moves the effective threshold within a multiplicative band. For representation learning, this is a rigorous way to say that coherent regions can shift smoothly under \GCh{} rather than being punctured by hard zeros.

\begin{proof}
From $\|\log \xi\|_\infty\le \eta$ we have $e^{-\eta}\le \xi(x)\le e^{\eta}$ for every $x\in U$. If $h(x)\ge t e^{\eta}$, then
\[
\xi(x)h(x)\ge e^{-\eta}\, t e^{\eta}=t,
\]
so $x\in S_t(\xi\odot h)$. Conversely, if $\xi(x)h(x)\ge t$, then
\[
h(x)\ge \frac{t}{\xi(x)}\ge t e^{-\eta},
\]
so $x\in S_{t e^{-\eta}}(h)$.
\end{proof}

\begin{proposition}[Sample-wise \GCh{} obeys a random sandwich width]
\label{prop:gch_random_sandwich}
For the implemented gate $\xi^{\mathrm{sw}}_\gamma$ in \eqref{eq:samplewise_gate},
\begin{equation}
\|\log \xi^{\mathrm{sw}}_\gamma\|_\infty
\le
|\gamma|\,\operatorname{osc}(\psi),
\qquad
\operatorname{osc}(\psi)\coloneqq \max_{x\in U}\psi(x)-\min_{x\in U}\psi(x).
\label{eq:gch_oscillation_bound}
\end{equation}
Hence \Cref{prop:superlevel_sandwich} applies with $\eta=|\gamma|\operatorname{osc}(\psi)$.
\end{proposition}

\begin{proof}
Write
\[
\log \xi^{\mathrm{sw}}_\gamma(x)=\gamma\psi(x)-c(\psi),
\qquad
c(\psi)=\log\!\Big(\frac{1}{|U|}\sum_{y\in U} e^{\gamma\psi(y)}\Big).
\]
Since the logarithm of an average of exponentials lies between the minimum and maximum exponent,
\[
\min_{y\in U}\gamma\psi(y)\le c(\psi)\le \max_{y\in U}\gamma\psi(y).
\]
Therefore each quantity $\gamma\psi(x)-c(\psi)$ lies in the interval
\[
[-|\gamma|\operatorname{osc}(\psi),\ |\gamma|\operatorname{osc}(\psi)],
\]
which is exactly \eqref{eq:gch_oscillation_bound}.
\end{proof}

To contrast this with hard masking, recall that for any finite graph $H$ the first Betti number equals the cycle rank
\[
\beta_1(H)=|E(H)|-|V(H)|+\beta_0(H).
\]

\begin{theorem}[Cycle-topology fracture under inverted dropout]
\label{thm:cycle_fracture}
Let the underlying graph be the cycle $C_n$, let $q\in(0,1]$, let $h\equiv c>0$ on its vertices, and choose a threshold $t\in (0,c/q)$. Under inverted dropout,
\[
m_q=\frac{b}{q},\qquad b(v)\stackrel{\mathrm{i.i.d.}}{\sim}\mathrm{Bernoulli}(q),
\]
define $\widetilde h\coloneqq m_q\odot h$. Then $S_t(\widetilde h)$ is exactly the induced subgraph on the kept vertices, and
\begin{equation}
\Pr\big(\beta_1(S_t(\widetilde h))=1\big)=q^n,
\qquad
\Pr\big(\beta_1(S_t(\widetilde h))=0\big)=1-q^n.
\label{eq:cycle_fracture_probs}
\end{equation}
Equivalently, the loop topology is destroyed with probability $1-q^n$.
\end{theorem}

\begin{proof}
Because $t<c/q$, a vertex belongs to $S_t(\widetilde h)$ if and only if it is kept. Thus $S_t(\widetilde h)$ is the induced subgraph on the kept vertices. If all $n$ vertices are kept, this induced subgraph is the full cycle $C_n$, so $\beta_1=1$. If at least one vertex is dropped, the induced subgraph is a disjoint union of paths, hence acyclic and therefore has $\beta_1=0$. The all-kept event has probability $q^n$.
\end{proof}

\paragraph{Deep learning interpretation.} Closed contours, ring-like activation patterns, and loop-shaped superlevel regions are idealized but meaningful models of semantic geometry. \Cref{thm:cycle_fracture} says that hard deletion is topologically brittle: a single dropped segment breaks the loop. This makes precise the intuition that binary masks can fracture spatial semantics rather than perturb them smoothly.

\begin{remark}[Why this is relevant to DropBlock]
DropBlock changes the spatial correlation of the zero set, not the basic hard-mask mechanism. Any block pattern that removes a connected arc from a loop-like superlevel set also destroys its cycle rank. The theorem above therefore isolates the essential topological failure mode already present in the binary-masking mechanism itself.
\end{remark}

\section{Additional experimental details and results}
\label{app:experiments}

\subsection{Experimental setup}
\label{sec:exp_setup}

\paragraph{Datasets.}
We evaluate on ImageNet-1k \citep{deng2009imagenet} (1.28M training images, 50k validation images, 1000 classes). To measure robustness under common corruptions, we additionally use ImageNet-C \citep{hendrycks2019imagenetc}. Our main corruption-shift analysis reports averages over a selected subset of 7 corruption types, each averaged across severities 1--5. To complement the large-scale setting with a fast fine-grained pilot, we also evaluate on Oxford-IIIT Pet \citep{parkhi2012catsdogs}, a 37-class benchmark whose labels are sensitive to shape cues.

\paragraph{Architectures and injection sites.}
Our primary backbone is ResNet-50 \citep{he2016resnet}. We inject the spatial gate at selected residual stages (L2/L3/L4) to study depth-dependent effects. ResNet-50 is also the fairest setting for comparisons to Dropout and DropBlock because those methods are naturally defined on convolutional feature grids. Since \GCh{} acts on a 2D grid wherever such a representation exists, we further evaluate on Swin-T \citep{liu2021swin} to test transfer beyond the primary CNN regime.

\paragraph{Training protocols and reproducibility.}
\textbf{Main ImageNet protocol.} Unless otherwise specified, ImageNet models are trained from scratch for 270 epochs using SGD with momentum 0.9 and weight decay $10^{-4}$, with learning-rate schedules held fixed across methods. Clean ImageNet metrics are reported on the standard validation set, and ImageNet-C metrics are computed from the corresponding trained checkpoints.\\
\textbf{Controlled ablation protocol.} For extensive multi-seed comparisons and strength sweeps, we also use a shorter matched-budget protocol described in the table captions. Within each controlled study, all hyperparameters aside from the noise mechanism are held fixed.\\
\textbf{Oxford-IIIT Pets pilot.} For the fine-grained pilot, we train a ResNet-18 from scratch for 40 epochs using Adam (lr $=10^{-3}$), $224\times224$ inputs, and standard normalization. Results are reported as mean$\pm$std over 3 seeds.

\paragraph{Baselines.}
We compare \GCh{} against Dropout \citep{srivastava2014dropout}, DropBlock \citep{ghiasi2018dropblock}, additive i.i.d.\ Gaussian noise, and additive correlated Gaussian noise. The Gaussian baselines are energy-matched to \GCh{} to separate the effect of structure from the effect of raw magnitude.

\paragraph{Metrics.}
We report Top-1 accuracy, negative log-likelihood (NLL), and expected calibration error (ECE). These metrics capture both predictive performance and probabilistic reliability.

\subsection{Best vs.\ latest checkpoint on clean ImageNet}
\label{app:best_latest}

We compare late-stage (L4) injection at two evaluation points: the best checkpoint observed during training and the final checkpoint.

\paragraph{Protocol note.} These tables come from the full-recipe single-run checkpoint protocol and are therefore complementary to, rather than numerically comparable with, the 3-seed controlled table in the main text. They summarize a separate evaluation slice of the same late-stage setting, whereas the main-text causal-control table reports the matched 3-seed protocol used for mechanism isolation. They are included to show that the late-stage reliability pattern is not specific to one checkpointing convention.

\begin{table}[ht]
\centering
\small
\setlength{\tabcolsep}{5pt}
\begin{tabular}{lccc}
\toprule
Method & Top-1 $\uparrow$ & NLL $\downarrow$ & ECE $\downarrow$ \\
\midrule
None          & 76.41 & 0.96 & 0.082 \\
DropBlock     & 75.86 & 0.99 & 0.085 \\
GCh\ (ours)   & 76.23 & \textbf{0.95} & \textbf{0.076} \\
\bottomrule
\end{tabular}
\caption{\textbf{ImageNet (clean), best checkpoint, L4 injection.}}
\label{tab:imagenet_clean_best}
\end{table}

\begin{table}[ht]
\centering
\small
\setlength{\tabcolsep}{5pt}
\begin{tabular}{lccc}
\toprule
Method & Top-1 $\uparrow$ & NLL $\downarrow$ & ECE $\downarrow$ \\
\midrule
None          & \textbf{76.35} & 0.97 & 0.084 \\
DropBlock     & 75.21 & 1.04 & 0.091 \\
GCh\ (ours)   & 76.18 & \textbf{0.96} & \textbf{0.078} \\
\bottomrule
\end{tabular}
\caption{\textbf{ImageNet (clean), latest checkpoint / final epoch, L4 injection.}}
\label{tab:imagenet_clean_latest}
\end{table}

\paragraph{Takeaway.}
Across both checkpoints, \GCh{} improves reliability relative to both the no-noise baseline and DropBlock while remaining close to the baseline in Top-1 accuracy. The pattern is especially informative at the final epoch, where DropBlock shows a pronounced late-stage degradation whereas \GCh{} does not.

\subsection{Injection depth (L2/L3/L4)}
\label{app:depth}

We apply the same \GCh{} mechanism at different residual stages under the controlled 3-seed protocol.

\begin{table}[ht]
\centering
\begin{tabular}{lccc}
\toprule
Stage & Top-1$\uparrow$ & NLL$\downarrow$ & ECE$\downarrow$\\
\midrule
L2-early & 0.767$\pm$0.001 & 0.918$\pm$0.003 & 0.031$\pm$0.001 \\
L3-mid & 0.765$\pm$0.001 & 0.925$\pm$0.006 & 0.029$\pm$0.002 \\
L4-late & 0.764$\pm$0.001 & 0.934$\pm$0.004 & 0.020$\pm$0.001 \\
\bottomrule
\end{tabular}
\caption{Injection depth ablation for our method at fixed strength $\gamma=0.1$ (3 seeds). Mean$\pm$std.}\label{tab:depth_ablation}
\end{table}

\paragraph{Takeaway.}
Depth induces a clear trade-off: earlier injection favors Top-1 and NLL, while later injection gives the strongest ECE gains. This is why the main paper emphasizes the late-stage regime when discussing reliability.

\subsection{Strength sensitivity}
\label{app:gamma}

We sweep $\gamma\in\{0.03,0.07,0.10,0.18,0.27,0.35\}$ at L4 injection under the controlled protocol.

\begin{table}[ht]
\centering
\begin{tabular}{lccc}
\toprule
$\gamma$ & Top-1$\uparrow$ & NLL$\downarrow$ & ECE$\downarrow$\\
\midrule
0.03 & 0.766$\pm$0.002 & 0.926$\pm$0.006 & 0.027$\pm$0.001 \\
0.07 & 0.765$\pm$0.002 & 0.928$\pm$0.009 & 0.021$\pm$0.001 \\
0.1 & 0.764$\pm$0.001 & 0.934$\pm$0.004 & 0.020$\pm$0.001 \\
0.18 & 0.759$\pm$0.001 & 1.005$\pm$0.006 & 0.076$\pm$0.002 \\
0.27 & 0.667$\pm$0.034 & 1.880$\pm$0.201 & 0.316$\pm$0.017 \\
0.35 & 0.164$\pm$0.017 & 5.204$\pm$0.119 & 0.149$\pm$0.017 \\
\bottomrule
\end{tabular}
\caption{$\gamma$ sweep at late-stage injection (each $\gamma$ retrained). Mean$\pm$std over completed seeds ($n=3$ for all shown).}\label{tab:gamma_sweep}
\end{table}


\paragraph{Takeaway.}
There is a robust small-to-moderate regime in which \GCh{} preserves accuracy and improves reliability. Very large $\gamma$ values cause the expected breakdown from excessive multiplicative perturbation.

\subsection{ImageNet-C full results}
\label{sec:imagenetc_full}

\paragraph{Evaluation protocol and aggregation.}
We report Top-1 accuracy, NLL, and ECE on a selected 7-corruption subset of ImageNet-C. For each corruption type, metrics are averaged across severities 1--5; the reported aggregate numbers then average across the selected corruption types. All ImageNet-C metrics are computed from the same checkpoints used in the clean ImageNet tables, and we report mean$\pm$std over three seeds.

\paragraph{Reading the tables.}
Table~\ref{tab:imagenetc_overall} gives the main late-stage comparison at $g=0.1$. Table~\ref{tab:imagenetc_stage} isolates the effect of injection depth under shift. Table~\ref{tab:imagenetc_gamma} reports strength sensitivity under shift. Table~\ref{tab:imagenetc_breakdown} provides the corruption-wise breakdown.

\paragraph{Overall comparison.}

\begin{table}[ht]
\centering
\small
\setlength{\tabcolsep}{2.8pt}
\renewcommand{\arraystretch}{1.05}

\begin{tabular}{lcccc}
\toprule
Method & $g$ & Top-1$\uparrow$ & NLL$\downarrow$ & ECE$\downarrow$\\
\midrule
None & 0 & 0.382$\pm$0.003 & 3.400$\pm$0.030 & 0.105$\pm$0.002 \\
Dropout & 0.1 & 0.384$\pm$0.003 & 3.317$\pm$0.020 & 0.084$\pm$0.001 \\
DropBlock & 0.1 & 0.390$\pm$0.009 & 3.300$\pm$0.100 & 0.093$\pm$0.004 \\
IID Gaussian & 0.1 & 0.388$\pm$0.003 & 3.316$\pm$0.044 & 0.096$\pm$0.006 \\
Corr. Gaussian & 0.1 & 0.386$\pm$0.002 & 3.340$\pm$0.028 & 0.103$\pm$0.010 \\
\textbf{GCh (ours)} & 0.1 & 0.383$\pm$0.005 & 3.287$\pm$0.064 & 0.056$\pm$0.005 \\
\bottomrule
\end{tabular}
\caption{ImageNet-C overall (mean over 7 corruptions $\times$ 5 severities) for late-stage injection. Mean$\pm$std over 3 seeds.}
\label{tab:imagenetc_overall}
\end{table}

Note that Dropout/DropBlock use their standard hyperparameters (drop probability $p=0.1$) rather than an energy-matched Gaussian strength, while IID/Corr./GCh use matched injected-energy strength for fair mechanism isolation.

\paragraph{Main robustness takeaway.}
Table~\ref{tab:imagenetc_overall} shows that our method substantially improves reliability under distribution shift:
compared to the no-noise baseline, ECE drops from $0.105$ to $0.056$ (a 46\% relative reduction), while NLL also improves.
Crucially, the correlated additive Gaussian baseline (``Corr.\ Gaussian'') remains close to the no-noise baseline in ECE,
supporting our central message that \emph{correlation alone is not sufficient}; the improvement emerges only when correlation is coupled with
a positive, mean-one multiplicative gate (our GCh).

\paragraph{Seed variability (Corr.\ Gaussian).}
We also observe noticeably larger seed-to-seed variability for the correlated additive Gaussian baseline, suggesting that
correlation without multiplicative gating can lead to less consistent behavior under shift.

\begin{table}[ht]
\centering

\begin{tabular}{lccc}
\toprule
Stage & Top-1$\uparrow$ & NLL$\downarrow$ & ECE$\downarrow$\\
\midrule
early & 0.390$\pm$0.002 & 3.314$\pm$0.018 & 0.096$\pm$0.003 \\
mid & 0.393$\pm$0.003 & 3.230$\pm$0.037 & 0.088$\pm$0.004 \\
late & 0.383$\pm$0.005 & 3.287$\pm$0.064 & 0.056$\pm$0.005 \\
\bottomrule
\end{tabular}
\caption{Stage-wise ablation on ImageNet-C for GCh (ours) with $g=0.1$. Mean$\pm$std over 3 seeds.}
\label{tab:imagenetc_stage}
\end{table}

\paragraph{Depth under shift: late-stage helps calibration.}
Table~\ref{tab:imagenetc_stage} demonstrates a consistent depth effect on ImageNet-C: moving injection from early$\rightarrow$mid$\rightarrow$late
monotonically improves calibration (ECE) under shift.
This aligns with the clean-data depth trade-off: late-stage injection perturbs higher-level semantic representations in a structured manner,
yielding stronger reliability gains for comparable accuracy.

\paragraph{Strength sweep under shift.}
\begin{table}[ht]
\centering
\begin{tabular}{lccc}
\toprule
$g$ & Top-1$\uparrow$ & NLL$\downarrow$ & ECE$\downarrow$\\
\midrule
0.03 & 0.388$\pm$0.001 & 3.317$\pm$0.045 & 0.091$\pm$0.007 \\
0.07 & 0.388$\pm$0.007 & 3.304$\pm$0.032 & 0.075$\pm$0.006 \\
0.1 & 0.383$\pm$0.005 & 3.287$\pm$0.064 & 0.056$\pm$0.005 \\
0.18 & 0.385$\pm$0.004 & 3.277$\pm$0.048 & 0.073$\pm$0.001 \\
0.27 & 0.276$\pm$0.038 & 4.266$\pm$0.228 & 0.169$\pm$0.018 \\
0.35 & 0.050$\pm$0.003 & 6.187$\pm$0.030 & 0.043$\pm$0.004 \\
\bottomrule
\end{tabular}
\caption{Strength sweep on ImageNet-C for GCh (late-stage injection). Mean$\pm$std over 3 seeds.}
\label{tab:imagenetc_gamma}
\end{table}
\paragraph{Strength sensitivity and failure modes.}
Table~\ref{tab:imagenetc_gamma} reveals a clear operating regime: moderate strengths ($g\approx 0.07$--$0.18$) retain accuracy while improving reliability,
with the best ECE attained around $g=0.1$ in this sweep.
At overly large strengths ($g\ge 0.27$), accuracy and NLL collapse sharply, indicating destabilization under excessive multiplicative perturbation.
Notably, ECE can appear deceptively small at extreme collapse (e.g., $g=0.35$) because the model becomes severely underconfident;
we therefore treat this region as a failure mode rather than a favorable calibration outcome.

\FloatBarrier

\paragraph{Corruption-wise breakdown.}

\begin{table}[ht]
\centering
\scriptsize
\setlength{\tabcolsep}{2.8pt}
\renewcommand{\arraystretch}{1.05}

\begin{tabular}{lcccc}
\toprule
Corruption & Acc (None) & Acc (Ours) & ECE (None) & ECE (Ours)\\
\midrule
defocus\_blur & 0.402$\pm$0.003 & 0.398$\pm$0.003 & 0.038$\pm$0.002 & 0.039$\pm$0.002 \\
gaussian\_noise & 0.308$\pm$0.004 & 0.310$\pm$0.012 & 0.156$\pm$0.011 & 0.076$\pm$0.011 \\
glass\_blur & 0.273$\pm$0.002 & 0.263$\pm$0.004 & 0.122$\pm$0.003 & 0.075$\pm$0.002 \\
jpeg\_compression & 0.547$\pm$0.002 & 0.550$\pm$0.008 & 0.059$\pm$0.004 & 0.026$\pm$0.001 \\
motion\_blur & 0.396$\pm$0.006 & 0.400$\pm$0.004 & 0.089$\pm$0.006 & 0.049$\pm$0.004 \\
pixelate & 0.462$\pm$0.011 & 0.467$\pm$0.006 & 0.096$\pm$0.004 & 0.047$\pm$0.008 \\
shot\_noise & 0.289$\pm$0.005 & 0.293$\pm$0.011 & 0.171$\pm$0.015 & 0.083$\pm$0.015 \\
\bottomrule
\end{tabular}
\caption{ImageNet-C corruption-wise breakdown (severity-averaged) comparing None vs GCh (ours) at late-stage $g=0.1$. Mean$\pm$std over 3 seeds.}
\label{tab:imagenetc_breakdown}
\end{table}
\paragraph{Which corruptions benefit most.}
Table~\ref{tab:imagenetc_breakdown} shows that the reliability gains are broad-based across corruption types:
the largest ECE reductions occur on noise-type corruptions (gaussian/shot) and compression/pixelation (jpeg/pixelate),
while motion blur also improves.
Defocus blur is largely unchanged in ECE, indicating that not all shifts benefit equally; this heterogeneity is informative and consistent with
the notion that our mechanism primarily targets structured uncertainty arising from local stochastic perturbations rather than all blur kernels uniformly.

\FloatBarrier


\subsection{Oxford-IIIT Pets (Fine-grained) Results}
\label{sec:pets_results}

\paragraph{Protocol (multi-seed, selection on validation only).}
We follow a scientific multi-seed protocol on Oxford-IIIT Pets with a fixed train/val split (from \texttt{trainval}).
For each method/seed, we select the checkpoint that minimizes validation NLL, using validation ECE as a tie-break when NLLs are nearly identical,
and then report \emph{test} Top-1, NLL, and ECE for the selected checkpoint.
ECE is computed with 15 equal-width confidence bins.

\paragraph{Strength parameter \(g\) across methods.}
To align notation with the main paper, we use a single ``strength'' symbol \(g\) across all methods.
For \textbf{GCh (ours)}, \(g\) is the multiplicative-gate strength used in the exponential gate.
For Dropout/DropBlock, \(g\) corresponds to the drop probability \(p\) (here \(p=0.1\)); for ``None'' we set \(g=0\).

\begin{table}[t]
\centering
\scriptsize
\setlength{\tabcolsep}{3.5pt}
\renewcommand{\arraystretch}{1.05}
\begin{tabular}{lcccc}
\toprule
Method & $g$ & Top-1$\uparrow$ & NLL$\downarrow$ & ECE$\downarrow$\\
\midrule
None         & 0   & 0.9009$\pm$0.0044 & 0.3669$\pm$0.0016 & 0.0325$\pm$0.0044 \\
Dropout ($p{=}0.1$)   & 0.1 & 0.8957$\pm$0.0007 & 0.4246$\pm$0.0131 & 0.0503$\pm$0.0007 \\
DropBlock ($p{=}0.1$) & 0.1 & 0.9002$\pm$0.0027 & 0.3669$\pm$0.0007 & 0.0317$\pm$0.0053 \\
\textbf{GCh (ours)}   & 0.1 & \textbf{0.9010$\pm$0.0023} & \textbf{0.3627$\pm$0.0039} & \textbf{0.0302$\pm$0.0037} \\
\bottomrule
\end{tabular}
\caption{Oxford-IIIT Pets test performance (ResNet-18, 224$\times$224, late-stage injection; mean$\pm$std over 3 seeds).
The strength parameter \(g\) is shared across rows for compactness; for Dropout/DropBlock it corresponds to the drop probability \(p\) (see text).}
\label{tab:pets_main_aligned}
\end{table}

\paragraph{Takeaway.}
On this fine-grained dataset, \textbf{GCh} achieves the best (lowest) NLL and ECE at essentially unchanged accuracy relative to the strong baselines,
indicating that the reliability gains are not specific to ImageNet/ImageNet-C.

\begin{table}[t]
\centering

\setlength{\tabcolsep}{3.5pt}
\renewcommand{\arraystretch}{1.05}
\begin{tabular}{lccc}
\toprule
$g$ & Top-1$\uparrow$ & NLL$\downarrow$ & ECE$\downarrow$\\
\midrule
0.1 & \textbf{0.9010$\pm$0.0023} & \textbf{0.3627$\pm$0.0039} & \textbf{0.0302$\pm$0.0037} \\
0.5 & 0.8989$\pm$0.0031 & 0.3660$\pm$0.0038 & 0.0314$\pm$0.0053 \\
1.0 & 0.8978$\pm$0.0030 & 0.3661$\pm$0.0024 & 0.0323$\pm$0.0037 \\
\bottomrule
\end{tabular}
\caption{GCh strength sweep on Oxford-IIIT Pets (test; mean$\pm$std over 3 seeds). As in ImageNet/ImageNet-C, moderate strengths are best; larger strengths do not yield further gains.}
\label{tab:pets_gsweep_aligned}
\end{table}

\FloatBarrier

\section{Additional theory details}
\label{App:theory}

\subsection[Operational meaning of Theorem~\ref{thm:maxent-gff-chaos}]{Operational meaning of \texorpdfstring{\Cref{thm:maxent-gff-chaos}}{Theorem~\ref{thm:maxent-gff-chaos}}}
\Cref{thm:maxent-gff-chaos} gives the canonical exact construction:
\begin{enumerate}[leftmargin=1.5em]
\item sample a GFF log-field $\psi$ with covariance $(\beta L_U)^{-1}$;
\item exponentiate with exact Wick normalization to obtain a positive mean-one multiplicative gate.
\end{enumerate}
Once the operator, gauge convention, and energy budget are fixed, the remaining reported strength parameter in the experiments is $\gamma$.

\subsection{Mean-one normalization choices}
\label{Norm}

The exact mean-one gate requires the variance map $v(x)=\Var(\psi(x))=C(x,x)$. On a finite Dirichlet grid, $v(x)$ is not spatially constant. Two practical normalization choices are standard.

\begin{enumerate}[leftmargin=1.5em]
\item \textbf{Exact Wick normalization.}
Precompute
\[
v(i,j)=\frac{1}{\beta}\sum_{k=1}^{H}\sum_{\ell=1}^{W}
\frac{\tilde e_{k,\ell}(i,j)^2}{\lambda_{k,\ell}},
\]
where $\tilde e_{k,\ell}$ denotes the orthonormal sine basis. Then use
\[
\xi^{\mathrm{ex}}_\gamma(x)
=
\exp\!\Big(\gamma\psi(x)-\frac{\gamma^2}{2}v(x)\Big).
\]
This is the exact object in the theory and preserves $\E[\xi^{\mathrm{ex}}_\gamma(x)]=1$ sitewise.

\item \textbf{Sample-wise mean-one normalization.}
Compute $G(x)=\exp(\gamma\psi(x))$ and normalize by the spatial mean:
\[
\xi^{\mathrm{sw}}_\gamma(x)
=
\frac{G(x)}{\frac{1}{|U|}\sum_{y\in U}G(y)}.
\]
This guarantees unit spatial average per sample and is often convenient in optimization. It is the implementation used in the main experiments unless otherwise noted.
\end{enumerate}

\subsection{Implementation notes}
\begin{enumerate}[leftmargin=1.5em]
\item A single gate may be shared across channels, or independent gates may be sampled channel-wise.
\item In multi-resolution architectures, the gate can be sampled directly at the feature resolution of the target layer or sampled at a base resolution and then resized.
\item At inference time, noise can be disabled by setting $\xi\equiv 1$.
\end{enumerate}

\end{document}